\documentclass[10pt,twocolumn,letterpaper]{article}

\usepackage[pagenumbers]{wacv}

\usepackage{bm}
\usepackage{amsthm}
\usepackage{algorithm}
\usepackage{algpseudocode}
\usepackage{placeins}
\usepackage{tabularx}
\usepackage{multirow}
\usepackage{colortbl}
\usepackage{tikz}
\usetikzlibrary{arrows.meta,positioning}

\newtheorem{theorem}{Theorem}
\newtheorem{proposition}[theorem]{Proposition}

\newcommand{\pguidance}{P-Guidance\xspace}
\newcommand{\method}{Step-PI\xspace}

\definecolor{wacvblue}{rgb}{0.21,0.49,0.74}
\usepackage[breaklinks,colorlinks,allcolors=wacvblue]{hyperref}
\hypersetup{
  pdftitle={Training-Free Inpainting Across Domains with a Frozen Text-to-Image Diffusion Model},
  pdfauthor={Zhenhuan Wang and Fengyi Yuan}
}

\title{Training-Free Inpainting Across Domains\\
with a Frozen Text-to-Image Diffusion Model}

\author{
Zhenhuan Wang$^{1}$ \qquad Fengyi Yuan$^{2}$\\
{\small $^{1}$School of Data Science, The Chinese University of Hong Kong, Shenzhen}\\
{\small $^{2}$School of Science and Engineering, The Chinese University of Hong Kong, Shenzhen}
}

\begin{document}
\maketitle

\begin{abstract}
We show that a frozen generic text-to-image diffusion model can perform
conditional inpainting across three evaluated natural-image domains with one
fixed controller configuration, without inpainting-specific weight training,
dataset-specific weight adaptation, or learned inpainting-specific conditioning
channels.  \method augments known-region projection with boundary–interior latent
feedback, persistent PI state, and a predefined four-field release schedule that
modulates controller signals along the reverse trajectory.  Developed only on
Main35-disjoint CelebA-HQ pilots, the controller transfers unchanged to AFHQ and
Places2.  Across two field-identical comparisons on the same 3,500 cases, adding
persistent state and replacing uniform release with the predefined schedule each
improve all 15 dataset--metric cells; 95\% bootstrap intervals exclude zero for
all five metrics in both comparisons.  In descriptive native-route comparisons,
\method leads LanPaint and PILOT---the closest evaluated training-free baselines
using vanilla SD1.5---on all five equal-dataset macro metrics.  Inpainting-trained
systems retain the absolute metric leads but rely on substantial
inpainting-specific offline optimization. Our method provides a complementary
approach for repurposing a frozen generic text-to-image model for cross-domain
inpainting through test-time latent control.
\end{abstract}

\begin{figure}[t]
  \centering
  \includegraphics[width=\columnwidth]{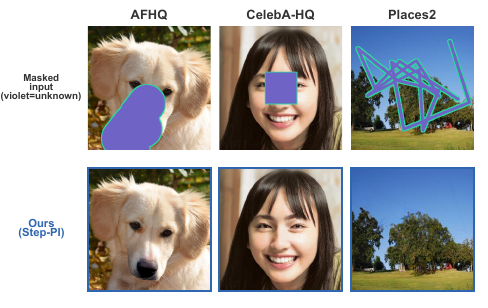}
  \caption{Three-domain conditional inpainting with frozen SD1.5.  The same
  configuration-locked Step-PI controller is applied to an outcome-independent
  random sample spanning AFHQ, CelebA-HQ, and Places2, without weight updates or
  learned inpainting-specific conditioning channels.  Text conditions are fixed before
  inference.  Opaque violet denotes unknown pixels, and the mint contour is
  display-only.}
  \label{fig:cross-domain-teaser}
\end{figure}
\section{Introduction}

Modern diffusion inpainters achieve strong quality through inpainting-specific weights or specialized conditioning interfaces
\citep{rombach2022high,ju2024brushnet,xu2025pixelhacker,liu2024prefpaint}.
Generic text-to-image models lack these inpainting-specific resources.  Nevertheless, we show that a frozen generic text-to-image model can be repurposed for conditional inpainting through test-time latent control, without weight updates or learned inpainting-specific conditioning channels.

Each case supplies an observed image, a binary mask, and a user-provided text
condition fixed before inference.  Known-region projection anchors observed
context on the latent grid, and final RGB compositing restores known pixels
exactly; neither constrains the generated region.  Its boundary requires local
continuity, whereas its interior requires longer-range contextual and structural
coherence, and the useful correction balance changes over the reverse
trajectory.  This motivates boundary–interior feedback, persistent state, and a predefined release schedule.

We introduce \method, a training-free closed-loop controller that operates during deterministic DDIM sampling with a frozen vanilla SD1.5 text-to-image backbone.  Boundary and interior objectives produce current-latent feedback signals, exponentially discounted states retain directional history, and the four fields of a predefined release schedule modulate already-computed controller signals before shared safeguards.  The backbone, objectives, gains, state update, and
release schedule are fixed across all three domains.  Text conditions are
fixed within controller comparisons; a separate matched prompt-presence intervention
changes only the positive text condition.  Field-identical \pguidance $\rightarrow$ PI and PI
$\rightarrow$ \method comparisons isolate persistent state and the complete predefined
release schedule, respectively.

Evidence follows three layers.  First, parameters developed on Main35-disjoint
CelebA-HQ pilots transfer unchanged to AFHQ and Places2, and a fixed
prompt-presence intervention shows responsiveness across all three domains.
Second, the two field-identical additions improve all 15 dataset--metric cells.
A protocol-stratified random-subset audit, conducted without inspecting outputs,
probes conditional schedule
sensitivity within the predefined schedule, while separate five-seed audits test
the persistent-state and joint-release additions across paired stochastic
inference initializations.  Third, under descriptive comparisons using audited native routes, \method records better values than the two closest evaluated vanilla-SD1.5 training-free baselines on all five equal-dataset macro metrics. Inpainting-trained systems retain the absolute metric leads but require substantial offline optimization to acquire inpainting-specific capability. \method avoids this training and any dataset-specific weight adaptation, although it shifts substantial computation to inference.

Our contributions are:
\begin{itemize}
  \item three-domain, configuration-locked evidence that one frozen generic
  text-to-image prior supports conditional inpainting on two transfer-only
  targets without learned inpainting-specific channels;
  \item \method, a PI-structured latent controller that combines
  boundary--interior current feedback, discounted trajectory memory, and
  a predefined release schedule; and
  \item a controlled evaluation that uses matched comparisons to isolate the effects of persistent state and release scheduling, tests robustness across paired inference seeds, and compares \method with external baselines under each method's native inference pipeline.
\end{itemize}

\section{Related Work}
\label{sec:related_work}

\subsection{Trained Inpainting Methods}

Trained inpainters acquire task capability through learned weights or
conditioning interfaces.  SD-Inpaint \citep{rombach2022high} uses an
inpainting-specific latent interface, BrushNet \citep{ju2024brushnet} adds a
decomposed dual-branch architecture, and PixelHacker \citep{xu2025pixelhacker} targets
structural and semantic consistency.  Broader diffusion reward- and
preference-based adaptation includes DDPO \citep{black2023ddpo}, DPOK
\citep{fan2023dpok}, and PrefPaint \citep{liu2024prefpaint}.  Inpainting-specific
training is distinct from dataset-specific weight adaptation, and a model
trained for inpainting need not be retrained for every evaluation domain.  Because these systems
contain learned inpainting capability absent from a generic text-to-image
backbone, we use them as strong quality references rather than matched
interventions.

\subsection{Training-Free Inpainting Methods}

``Training-free'' methods keep deployed weights fixed, but differ in foundation,
sampler, and conditioning interface.  LanPaint \citep{zheng2025lanpaint}
combines Langevin dynamics with ODE-based diffusion, whereas PILOT
\citep{pan2024coherent} optimizes the latent during DDIM sampling; both operate
on vanilla SD1.5 without inpainting-trained weights and are our closest evaluated baselines using the same pretrained diffusion model, although their native routes differ.  Other
approaches modify initialization (InverFill \citep{vu2026inverfill}),
attention/style (HarmonPaint \citep{li2025harmonpaint}), resampling (RePaint
\citep{lugmayr2022repaint}), auxiliary propagation (LatentPaint
\citep{corneanu2024latentpaint}), or pretrained inpainting interfaces (FreeCond
\citep{hsiao2026freecond}).  GradPaint \citep{grechka2024gradpaint} backpropagates
masked MSE and, for pixel-space models, boundary losses; DING
\citep{moufad2026ding} derives VJP-free Gaussian posterior transitions for
zero-shot latent inpainting; and HiGS
\citep{sadat2026higs} uses denoiser-prediction history for generic sampling
enhancement.  Broader test-time control spans posterior gradients, differentiable
guidance, trajectory optimization, calibration, and reward-aware or
stochastic-control formulations
\citep{chung2023dps,bansal2023universal,wallace2023doodl,geyfman2026calibrated,kim2025testtime,song2021score,zhang2022pis,vargas2023dds,pandey2025variational}.

These heterogeneous interfaces prevent component-level causal attribution
across methods.  \method instead constructs normalized task-specific
boundary--interior feedback from the observed context, mask, and current decoded
estimate on generic SD1.5.  Its
field-identical ladder isolates persistent trajectory memory (\pguidance
$\rightarrow$ PI) and the joint effect of a predefined release schedule (PI $\rightarrow$
\method), while LanPaint and PILOT provide same-foundation external positioning.

\section{Preliminaries: Sequential-Control View}
\label{sec:preliminaries}

\subsection{Latent Diffusion and DDIM Sampling}
\label{sec:prelim-ddim}

A latent diffusion model learns to reverse a noising process in a compressed
latent space \citep{ho2020denoising,rombach2022high}.  Under the standard DDPM
parameterization,
\begin{align*}
z_t &= \sqrt{\bar\alpha_t}\,z_0
      +\sqrt{1-\bar\alpha_t}\,\epsilon,\\
\widehat\epsilon_t &= \epsilon_\theta(z_t,t,c),\\
\widehat z_0 &=
\frac{z_t-\sqrt{1-\bar\alpha_t}\,\widehat\epsilon_t}
     {\sqrt{\bar\alpha_t}},\\
z_{t'}^{\mathrm{DDIM}} &=
\sqrt{\bar\alpha_{t'}}\,\widehat z_0+
\sqrt{1-\bar\alpha_{t'}}\,\widehat\epsilon_t,
\qquad t'<t .
\end{align*}
DDIM reuses the pretrained noise predictor to define a non-Markovian reverse
trajectory \citep{song2020denoising}.  We use deterministic DDIM ($\eta=0$),
which adds no per-step noise and makes the field-identical feedback comparisons
controlled and easier to interpret.  We use it as a controlled inference substrate
without claiming superiority over stochastic DDPM.

Stable Diffusion instantiates latent diffusion with a text-conditioned U-Net
operating in VAE latent space \citep{rombach2022high}.  We use the pretrained vanilla Stable Diffusion v1.5 text-to-image model and keep all of its parameters frozen; the fixed text condition $c$ enters the frozen denoising transition. We use neither an inpainting-specific checkpoint nor any learned inpainting-specific conditioning channels.

\subsection{PI-Inspired Stateful Feedback}
\label{sec:prelim-pi-feedback}

Known-region projection anchors observed context on the latent grid, and final
RGB compositing restores visible pixels exactly; neither constrains seam quality
or the generated interior trajectory.  Classical PI combines current and
accumulated error \citep{astrom2008feedback},
\begin{align*}
s_n &= s_{n-1}+e_n,\\
u_n &= K_P e_n+K_I s_n .
\end{align*}
We borrow only this decomposition into current feedback and historical state.  In our finite-horizon, time-varying reverse process, the decoded clean estimate serves as the controller observation. Bounded, normalized boundary/interior directions provide instantaneous feedback; their clipped, exponentially discounted histories form persistent states; and the masked latent correction serves as the control action.  Thus, ``PI-structured'' denotes current
normalized feedback plus revisable directional history, not raw-error
proportionality or exact stationary-error integration.  PI and Step-PI denote the evaluated variants in Sec.~\ref{sec:method}, which instantiate this current-plus-history structure with spatial weighting, normalization, clipping, and uniform or predefined nonuniform release, respectively. The labels describe this structural connection rather than a textbook linear-control formulation.

\subsection{Finite-Horizon Control Lens}
\label{sec:sequential-control-view}

At reverse step $k$, known-context projection precedes the frozen
text-conditioned DDIM transition, which returns a decoded clean estimate and an
uncontrolled next latent.  The clean estimate, image, mask, persistent states,
and predeclared release determine a bounded action, applied at the
post-transition latent port before the next projection.  Thus feedback is formed
from the current observation while actuation occurs after transition; this
execution order neither identifies current-latent feedback with a next-state
objective gradient nor assumes cross-timestep gradient transport.
Section~\ref{sec:method-overview} gives the exact interface.

\section{Methodology}
\label{sec:method}

Building on Section~\ref{sec:preliminaries}, \method repurposes frozen
text-to-image diffusion for an observed RGB image $y$, unknown-region mask $M$
($M=1$ to generate and $M=0$ to preserve), and fixed text condition $c$ through
known-region projection, image- and mask-aware boundary--interior feedback,
persistent state, and a predefined release schedule, without weight updates or
learned inpainting-specific conditioning channels.

\subsection{Problem Setup and Closed-Loop Interface}
\label{sec:method-overview}
\label{sec:method-base}
\label{sec:control-interface}

All internal controls use a frozen vanilla SD1.5 text-to-image backbone with deterministic DDIM sampling.  Let
$M_z$ and $K=1-M$ denote the latent unknown-region and image known-region masks.
At scheduler state $t_k$, the reverse latent $z^k$ runs from initial noisy $z^N$
to terminal $z^0$, with $t_N>\cdots>t_0$.  Boundary/interior memories are
initialized once at zero and enter step $k$ as $\xi_B^{k+1},\xi_I^{k+1}$,
while $\rho_k$ denotes the release vector.  The memory recurrences and release
construction are defined in
Sections~\ref{sec:pi-controller}--\ref{sec:selective-release}.  For
$k=N,\ldots,1$, projection $\Pi_t$, frozen transition
$\mathcal D_{t_k\rightarrow t_{k-1}}$, and controller $\mathcal C_k$ form the
closed-loop interface
\begin{align}
\widetilde z^k &= \Pi_{t_k}(z^k), \nonumber\\
(\widehat z_0^k,z_{\mathrm{base}}^{k-1})
  &= \mathcal D_{t_k\rightarrow t_{k-1}}(\widetilde z^k,c), \nonumber\\
(a_k,\xi_B^k,\xi_I^k)
  &= \mathcal C_k(\widehat z_0^k,y,M,\xi_B^{k+1},\xi_I^{k+1};\rho_k), \nonumber\\
z^{k-1} &= \Pi_{t_{k-1}}(z_{\mathrm{base}}^{k-1}+a_k).
\label{eq:method-overview}
\end{align}
The fixed condition $c$ enters only the frozen transition; $y$ and $M$ define
projection and feedback, and $a_k$ affects only unknown-region coordinates
before the next projection.  Feedback is differentiated at the current
projected latent $\widetilde z^k$; after state update, release, and safeguards,
$a_k$ is applied to post-transition $z_{\mathrm{base}}^{k-1}$ on the same fixed
VAE lattice.  This is causal one-pass actuation, not cross-timestep gradient
transport, $-\nabla_{z^{k-1}}\mathcal L$, or next-state descent.

Known-region projection uses a leakage-free clean reference latent
$z_0^{\mathrm{known}}$ constructed only from visible pixels:
\begin{equation}
\Pi_t(z)=(1-M_z)\odot q_t^{\mathrm{known}}+M_z\odot z.
\label{eq:known-projection}
\end{equation}
Here $\odot$ denotes elementwise multiplication.  Equation~\eqref{eq:known-projection}
anchors known coordinates of the resized VAE latent grid, but VAE decoding is
not pixelwise; exact known-region RGB recovery therefore comes from final
compositing, not latent projection alone.  Reference construction, CFG prediction, DDIM transitions, VAE feedback decoding, mask resizing, and exact inference settings are provided in Supplementary Appendices~A and~B.

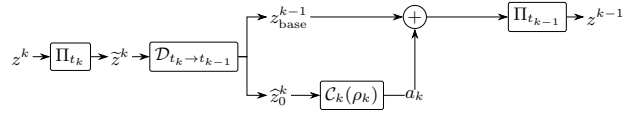
\begin{figure}[t]
  \centering
  \resizebox{0.98\columnwidth}{!}{%
  \begin{tikzpicture}[
    state/.style={font=\small, inner sep=1pt},
    action/.style={font=\small, inner sep=0pt, outer sep=0.25pt},
    op/.style={draw, rounded corners=1pt, font=\small, inner sep=2.5pt},
    add/.style={circle, draw, font=\small, inner sep=0pt, minimum size=4mm},
    arrow/.style={-{Stealth[length=1.7mm,width=1.1mm]}, line width=0.35pt}
  ]
    \node[state] (zk) {$z^k$};
    \node[op, right=3mm of zk] (p0) {$\Pi_{t_k}$};
    \node[state, right=3mm of p0] (zt) {$\widetilde z^k$};
    \node[op, right=3mm of zt] (d) {$\mathcal D_{t_k\to t_{k-1}}$};
    \node[state, above right=2mm and 5mm of d] (base) {$z_{\rm base}^{k-1}$};
    \node[state, below right=2mm and 5mm of d] (clean) {$\widehat z_0^k$};
    \node[op, right=5mm of clean] (c) {$\mathcal C_k(\rho_k)$};
    \node[add, right=16mm of base] (sum) {$+$};
    \node[op, right=14mm of sum] (p1) {$\Pi_{t_{k-1}}$};
    \node[state, right=3mm of p1] (next) {$z^{k-1}$};
    \draw[arrow] (zk)--(p0); \draw[arrow] (p0)--(zt);
    \draw[arrow] (zt)--(d);
    \draw[arrow] (d.east)--++(1.5mm,0)|-(base.west);
    \draw[arrow] (d.east)--++(1.5mm,0)|-(clean.west);
    \draw[arrow] (base)--(sum); \draw[arrow] (clean)--(c);
    \path (c.east -| sum.south) coordinate (actioncorner);
    \node[action] (actionlabel) at (actioncorner) {$a_k$};
    \draw[line width=0.35pt] (c.east)--(actionlabel.west);
    \draw[arrow] (actionlabel.north)--(sum.south);
    \draw[arrow] (sum)--(p1); \draw[arrow] (p1)--(next);
  \end{tikzpicture}%
  }
  \caption{Closed-loop frozen text-to-image inpainting interface.  Text enters
  the frozen transition, while the image and mask determine projection and
  feedback.  Step-PI and PI share this interface and differ only in $\rho_k$.
  Arrows mark computation and actuation locations, not gradient transport
  through the DDIM transition.}
  \label{fig:control-interface}
\end{figure}

\subsection{Boundary--Interior Feedback}
\label{sec:control-signals}

From the visible context and current clean estimate, the controller separates
local boundary stitching from broader interior continuation.  Let
$\mathcal L_B,\mathcal L_I$ be the corresponding scalar image-space objectives
and $g_B^k,g_I^k$ their bounded, normalized, objective-derived feedback
directions, differentiated at the current projected latent $\widetilde z^k$;
every controlled variant computes both objectives and directions at each step:
\begin{equation}
\widetilde z^k\rightarrow\widehat z_0^k\rightarrow\widehat x_0^k
\rightarrow
\begin{Bmatrix}\mathcal L_B\\\mathcal L_I\end{Bmatrix}
\rightarrow
\begin{Bmatrix}g_B^k\\g_I^k\end{Bmatrix}.
\label{eq:feedback-pipeline}
\end{equation}

With fixed weights $\lambda_j^B$, the boundary objective is
\begin{equation}
\mathcal L_B=
\lambda_1^B\mathcal L_{\mathrm{known}}+
\lambda_2^B\mathcal L_{\mathrm{pair}}+
\lambda_3^B\mathcal L_{\mathrm{TV}}+
\lambda_4^B\mathcal L_{\mathrm{boundary\text{-}grad}},
\label{eq:boundary-objective}
\end{equation}
Its four terms combine known-context RGB fidelity, seam pairing,
support-masked image variation, and RGB-gradient matching.  With fixed weights
$\lambda_j^I$, the interior objective is
\begin{equation}
\mathcal L_I=
\lambda_1^I\mathcal L_{\mathrm{lowfreq}}+
\lambda_2^I\mathcal L_{\mathrm{interior}}+
\lambda_3^I\mathcal L_{\mathrm{ring}}+
\lambda_4^I\mathcal L_{\mathrm{frequency}}.
\label{eq:interior-objective}
\end{equation}
Its terms encourage multiscale, depth-weighted contextual and structural
coherence.  Semantics come from the frozen text-conditioned transition; the
controller regularizes spatial and contextual consistency rather than directly
optimizing text alignment.  Exact definitions, conventions, and shared weights
are provided in Supplementary Appendix~A.

Here $D_z$ increases with depth in the latent mask and
$\operatorname{Normalize}_2$ is stabilized full-tensor $\ell_2$
normalization; resizing and stabilization details are provided in Supplementary Appendix~A:
\begin{align}
g_B^k&=\operatorname{Normalize}_2
[-M_z\odot\nabla_{\widetilde z^k}\mathcal L_B],\nonumber\\
g_I^k&=\operatorname{Normalize}_2
[-D_z\odot\nabla_{\widetilde z^k}\mathcal L_I].
\label{eq:normalized-gradients}
\end{align}
Normalization removes raw-scale differences across objectives and reverse
steps but does not make the directions timestep-invariant; fixed gains, release
factors, norm caps, and the final clamp set action magnitude.  It therefore
provides fixed calibration, not raw-magnitude adaptivity.

\subsection{Persistent PI-Structured State}
\label{sec:pi-controller}

Following Sec.~\ref{sec:prelim-pi-feedback}, $g_j^k$, $j\in\{B,I\}$, is the
current (P) branch and $\xi_j^k$ its discounted-history (I) branch.  PI and
Step-PI carry separate zero-initialized boundary/interior states through the
reverse trajectory.

To separate seam correction from interior continuation, we split $M_z$ into an
inner-boundary band $S_z$ and deep interior $I_z$.  $W_B$ emphasizes $S_z$
while retaining weaker $I_z$ feedback, and $\beta_B$ gives distinct retention
on boundary and deep-interior support.  Interior retention is uniform
($\gamma_I$) because $g_I^k$ is already depth-weighted by $D_z$.  Let
$(m_B,w_B)=(\beta_B,W_B)$ and $(m_I,w_I)=(\gamma_I,1)$.  Define
$\mathcal P_B=\operatorname{ClipNorm}_{\nu_B}$ as the full-tensor projection
onto the radius-$\nu_B$ $\ell_2$ ball.  For the effective recurrence, set
$\mathcal P_I(v)=v$ because Appendix~A shows algebraically that the implemented
interior radius-one norm guard acts as the identity under the evaluated
configuration.  The effective recurrence is
\begin{equation}
\xi_j^k=\mathcal P_j\!\left[
m_j\odot\xi_j^{k+1}+(1-m_j)\odot(w_j\odot g_j^k)\right].
\label{eq:pi-memory}
\end{equation}
\begin{proposition}[Effective persistent-state characterization]
\label{prop:production-state-characterization}
Let $d_j^k=w_j\odot g_j^k$,
$\mathcal C_B=\{\xi:\lVert\xi\rVert_2\le\nu_B\}$, and
$\mathcal C_I=\mathbb R^{4\times64\times64}$.  For elementwise
$0\le m_j\le1$, the update in Eq.~\eqref{eq:pi-memory} is the unique minimizer
\begin{equation}
\underset{\xi\in\mathcal C_j}{\operatorname{argmin}}\quad
\frac{1}{2}\lVert\sqrt{m_j}\odot(\xi-\xi_j^{k+1})\rVert_2^2+
\frac{1}{2}\lVert\sqrt{1-m_j}\odot(\xi-d_j^k)\rVert_2^2.
\label{eq:production-state-characterization}
\end{equation}
If $d_j^k\in\mathcal C_j$, it also obeys
$\lVert\xi_j^k-d_j^k\rVert_2\le
\lVert m_j\rVert_\infty\lVert\xi_j^{k+1}-d_j^k\rVert_2$.
\end{proposition}
Under the stated bounds, Proposition~\ref{prop:production-state-characterization}
is a per-step characterization of the implemented bounded compromise between
retained and current directions.  Since $d_j^k$ changes with $k$, it implies
neither trajectory-level contraction nor stability, and it does not establish
task benefit.  PI and Step-PI share this recurrence and differ only in the release schedule.
Matched \pguidance differs from PI only by setting $\xi_B^k=\xi_I^k=0$, so
P-Guidance$\rightarrow$PI is the strict persistent-state-only comparison and
the field-identical empirical test of persistent-state value.  Supplementary
Appendix~A gives the proof, spatial mappings, constants, and
trajectory-state details.

\subsection{Predefined Release Schedule and Matched Controls}
\label{sec:selective-release}
\label{sec:algorithm}
\label{sec:method-lineage}

Step-PI uses one predefined four-field release schedule: boundary-integral
$r_B$ and final-interior $r_I$ depend on $\bar\alpha_{t_k}$, while common-interior
$q_k$ and additional-memory $h_k$ depend on reverse-step progress $p_k$.
On the fixed 50-step DDIM route, both deterministically and monotonically index
the same trajectory, not independent state observations.  The fields are
predefined---not learned or uncertainty-derived---and modulate signals as
$\rho_k=(\rho_{B,k},\rho_{P,k},\rho_{H,k},\rho_{O,k})$; exact functions are provided in Supplementary Appendix~A.  For $j\in\{B,I\}$, the shared proportional, integral, and fused
actions are
\begin{equation}
\begin{aligned}
(\alpha_{P,k}^B,\alpha_{I,k}^B,\alpha_{O,k}^B)
  &=(1,\rho_{B,k},1),\\
(\alpha_{P,k}^I,\alpha_{I,k}^I,\alpha_{O,k}^I)
  &=(\rho_{P,k},\rho_{P,k}\rho_{H,k},\rho_{O,k}),\\
P_j^k&=\alpha_{P,k}^jK_P^jg_j^k,\\
I_j^k&=\operatorname{CapNorm}\!\left(
\alpha_{I,k}^jK_I^j\xi_j^k,c_j\lVert P_j^k\rVert_2\right),\\
u_j^k&=\alpha_{O,k}^j(P_j^k+I_j^k),\\
a_k(\rho_k)&=\operatorname{Clamp}_{[-a_{\max},a_{\max}]}
\!\left[M_z\odot(u_B^k+u_I^k)\right],\\
\rho_k^{\mathrm{PI}}&=(1,1,1,1),\\
\rho_k^{\mathrm{Step\text{-}PI}}
  &=(r_B(t_k),q_k,h_k,r_I(t_k)).
\end{aligned}
\label{eq:release-vector}
\end{equation}
Here $\operatorname{CapNorm}(v,r)=
v\min\{1,r/(\lVert v\rVert_2+10^{-8})\}$ is applied per sample over the full
latent tensor; the final clamp bounds the masked action.  PI sets all release
fields to one.  In the interior branch, $q_k$ precedes CapNorm and $r_I(t_k)$
follows it, yielding approximate pre-safeguard scale $q_kr_I(t_k)$ up to the
stabilizer and downstream safeguards; hence the four stored fields are not four
independent effective degrees of freedom.  Their timing roles are design
rationales, not optimality claims.  Release affects actuation before shared
safeguards, not the per-step feedback and state updates; final action need not
scale linearly, and history remains revisable (Fig.~\ref{fig:release-policies}).

The complete predefined release schedule was selected using Main35-disjoint
CelebA-HQ pilots and frozen before evaluation.  Section~\ref{sec:experiments}
evaluates the frozen schedule and reports a one-field-uniformization audit of
conditional sensitivity, without claims of minimality, individual or global
optimality, or dominance over shared-scalar or alternative schedules.

\begin{figure*}[t]
  \centering
  \begin{minipage}[t]{0.65\textwidth}
    \vspace{0pt}
    \includegraphics[width=\linewidth,trim=0 78bp 168bp 0,clip]{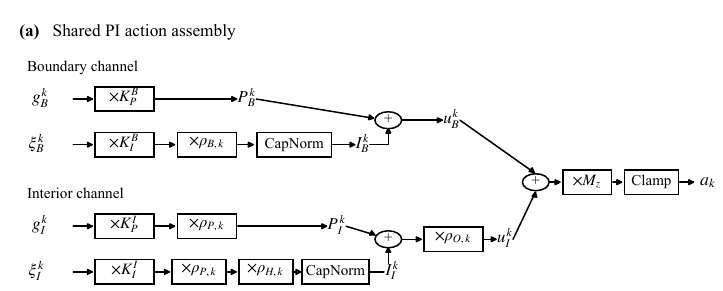}
    \vspace{-0.4ex}

    {\scriptsize Feedback and states update every step; $\rho_k$ modulates component inputs before shared safeguards.\par}
  \end{minipage}\hfill%
  \begin{minipage}[t]{0.33\textwidth}
    \vspace{0pt}
    \includegraphics[width=\linewidth]{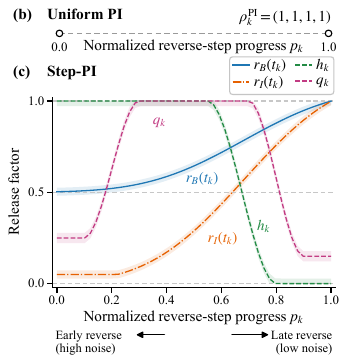}
  \end{minipage}
  \caption{Stateful boundary--interior action and release.  (a) Both feedback
  branches and persistent states update at every step; release modulates their
  contributions before the shared nonlinear safeguards.  (b) PI uses uniform
  release.  (c) Step-PI uses the four predefined fields of its release schedule
  along the reverse trajectory.}
  \label{fig:release-policies}
\end{figure*}

Algorithm~\ref{alg:step-pi} instantiates the matched construction ladder with
release policy $R$: DDIM-Proj$\rightarrow$P-Guidance adds boundary--interior
feedback, P-Guidance$\rightarrow$PI adds persistent state, and
PI$\rightarrow$Step-PI changes only $R$, from \textsc{UniformRelease} to
\textsc{StepRelease}.

\begin{algorithm}[t]
\caption{Field-identical PI/Step-PI control for one trajectory.}
\label{alg:step-pi}
\begin{algorithmic}[1]
\Require Image $y$, mask $M$, condition $c$, frozen diffusion sampler,
release policy $R$, controller hyperparameters $\Theta$
\Ensure Inpainted image with exact known-pixel preservation
\State Construct $M_z$, $z_0^{\mathrm{known}}$, and fixed known noise
$\epsilon^{\mathrm{known}}$
\State Order scheduler states as $t_N>t_{N-1}>\cdots>t_0$
\State Initialize $z^N$ from the paired seed; $\xi_B^{N+1}\gets0$,
$\xi_I^{N+1}\gets0$
\For{$k=N,\ldots,1$}
  \State $\widetilde z^k\gets\Pi_{t_k}(z^k)$
  \State $\widehat\epsilon^k\gets\Call{FrozenUNetCFG}{\widetilde z^k,t_k,c}$
  \State $(z_{\mathrm{base}}^{k-1},\widehat z_0^k)
  \gets\Call{DDIMStep}{\widehat\epsilon^k,\widetilde z^k,
  t_k\!\rightarrow t_{k-1}}$
  \State $\widehat x_0^k\gets\Call{VAEDecodeFP32}{\widehat z_0^k/s_{\mathrm{VAE}}}$
  \State Compute $\mathcal L_B,\mathcal L_I$ and current-latent feedback
  signals $g_B^k,g_I^k$
  \State Update $\xi_B^k,\xi_I^k$ from $\xi_B^{k+1},\xi_I^{k+1}$ by
  Eq.~\ref{eq:pi-memory}
  \State $\rho_k\gets R(t_k,k,N)$
  \State Construct $u_B^k,u_I^k$ by Eq.~\ref{eq:release-vector}
  \State $a_k\gets a_k(\rho_k)$ by Eq.~\ref{eq:release-vector}
  \State $z^{k-1}\gets\Pi_{t_{k-1}}(z_{\mathrm{base}}^{k-1}+a_k)$
\EndFor
\State Decode $z^0$ and composite exactly with $y$ on $(1-M)$
\State \Return composited RGB image
\end{algorithmic}
\end{algorithm}

\section{Experiments}
\label{sec:experiments}

We evaluate training-free Step-PI with a frozen generic SD1.5 text-to-image
(T2I) backbone on AFHQ, CelebA-HQ, and Places2, covering configuration-locked
cross-domain transfer without retuning, controlled ablations of controller
components, comparisons with external baselines, and inference cost.

\subsection{Evaluation and Transfer Protocol}
\label{sec:main35-contract}

\paragraph{Evaluation set and frozen inputs.}
Main35 comprises 3,500 paired cases from AFHQ v2
\citep{choi2020starganv2} (1,300), CelebA-HQ
\citep{lee2020maskgan} (1,400), and the Places365 Standard validation split
\citep{zhou2018places} (800; denoted \emph{Places2}), across 35 fixed 100-case
mask protocols.  Image, mask, prompt, and seed are held fixed within each
comparison; unknown-region ground truth is evaluation-only.  Prompts are fixed
without output-based optimization; preprocessing and sampling details are
provided in Supplementary Appendix~B.

\paragraph{Internal and transfer contracts.}
DDIM-Proj, P-Guidance, PI, and Step-PI share the frozen generic SD1.5
backbone \citep{rombach2022high}, projected image--mask interface, and paired
inputs.  Controller gains and release schedules were developed only on
Main35-disjoint CelebA-HQ pilots and then frozen; AFHQ and Places2 are
transfer-only targets without retuning, so the claim is limited to these two
domains.

\paragraph{External baselines.}
LanPaint \citep{zheng2025lanpaint} and PILOT \citep{pan2024coherent} are the
evaluated same-foundation training-free references: all three use vanilla
SD1.5 without inpainting-trained weights, learned inpainting interfaces, or
dataset-specific weight adaptation. GradPaint \citep{grechka2024gradpaint}
does not report an SD1.5 checkpoint, DING \citep{moufad2026ding} uses SD3.5,
and HiGS \citep{sadat2026higs} does not evaluate inpainting. Accordingly, we do not include these three methods
in our quantitative external comparison. SD-Inpaint \citep{rombach2022high}, BrushNet
\citep{ju2024brushnet}, and PixelHacker \citep{xu2025pixelhacker} form the
trained lane; PixelHacker is Places2-adapted. Native-interface comparisons are
descriptive; Supplementary Appendix~B audits attributes.

\paragraph{Metrics and statistics.}
Masked L1 measures held-out unknown-region RGB reconstruction, while
composite-based Masked LPIPS applies LPIPS-Alex \citep{zhang2018unreasonable}
to the full-frame output after exact known-region compositing.  Boundary
L1 and Boundary LPIPS measure seam fidelity, and CLIP-Q is an auxiliary CLIP-IQA-style
frozen two-prompt no-reference quality proxy
\citep{radford2021clip,wang2023exploringclip}; lower is better except for
CLIP-Q.  The controller neither sees unknown-region ground truth nor
optimizes LPIPS or CLIP-Q; its context and seam objectives align most directly
with L1 continuity.  These proxies do not replace human evaluation, and
reference scores may penalize plausible alternatives.  We average within
protocols and weight protocols and datasets equally.  Case-within-protocol
intervals are conditional on the fixed 35 protocols, not unseen mask families.
Full definitions and statistics are provided in Supplementary Appendix~B.

\begin{table*}[!t]
\centering
\small
\setlength{\tabcolsep}{7pt}
\renewcommand{\arraystretch}{1.04}
\begin{tabular}{@{}llccccc@{}}
\toprule
\multicolumn{7}{l}{\textbf{(a) Absolute Main35 results}} \\
\addlinespace[2pt]
Dataset & Method & Masked L1 $\downarrow$ & Boundary L1 $\downarrow$ & Masked LPIPS $\downarrow$ & Boundary LPIPS $\downarrow$ & CLIP-Q $\uparrow$ \\
\midrule
\multirow{4}{*}{AFHQ}
 & DDIM-Proj  & 0.1916 & 0.0442 & 0.1486 & 0.1177 & 0.6899 \\
 & P-Guidance & 0.1801 & 0.0413 & 0.1448 & 0.1112 & 0.6969 \\
 & PI         & 0.1768 & 0.0408 & 0.1441 & 0.1096 & 0.6990 \\
 & \textbf{Step-PI} & \textbf{0.1694} & \textbf{0.0320} & \textbf{0.1379} & \textbf{0.0980} & \textbf{0.7109} \\
\addlinespace[2pt]
\multirow{4}{*}{CelebA-HQ}
 & DDIM-Proj  & 0.1611 & 0.0376 & 0.0770 & 0.1209 & 0.5117 \\
 & P-Guidance & 0.1516 & 0.0348 & 0.0748 & 0.1143 & 0.5154 \\
 & PI         & 0.1449 & 0.0316 & 0.0727 & 0.1086 & 0.5219 \\
 & \textbf{Step-PI} & \textbf{0.1278} & \textbf{0.0236} & \textbf{0.0690} & \textbf{0.0923} & \textbf{0.5719} \\
\addlinespace[2pt]
\multirow{4}{*}{Places2}
 & DDIM-Proj  & 0.1843 & 0.0472 & 0.1398 & 0.1441 & 0.4558 \\
 & P-Guidance & 0.1752 & 0.0438 & 0.1370 & 0.1371 & 0.4581 \\
 & PI         & 0.1738 & 0.0431 & 0.1366 & 0.1351 & 0.4652 \\
 & \textbf{Step-PI} & \textbf{0.1606} & \textbf{0.0317} & \textbf{0.1300} & \textbf{0.1184} & \textbf{0.4763} \\
\midrule
\multicolumn{7}{l}{\textbf{(b) Mean direction-normalized gains for matched mechanism additions}} \\
\addlinespace[2pt]
Comparison & Cells won & Masked L1 & Boundary L1 & Masked LPIPS & Boundary LPIPS & CLIP-Q \\
\midrule
PI vs. P-Guidance   & 15/15 & +2.4\% & +3.9\% & +1.2\% & +2.6\% & +1.0\% \\
Step-PI vs. PI      & 15/15 & +7.8\% & +24.4\% & +4.7\% & +12.6\% & +4.6\% \\
\bottomrule
\end{tabular}
\caption{Main35 internal results.  (a) Equal-protocol aggregates; Step-PI
entries are shown in bold.  (b) Equal-dataset mean direction-normalized gains
for the field-identical persistent-state-only and release-only additions;
positive values favor the first named method, and ``Cells won'' counts
favorable dataset--metric cells.}
\label{tab:main35-internal-absolute}
\label{tab:main35-step-pi-relative-gain}
\end{table*}

\begin{figure*}[t]
  \centering
  \includegraphics[width=\textwidth]{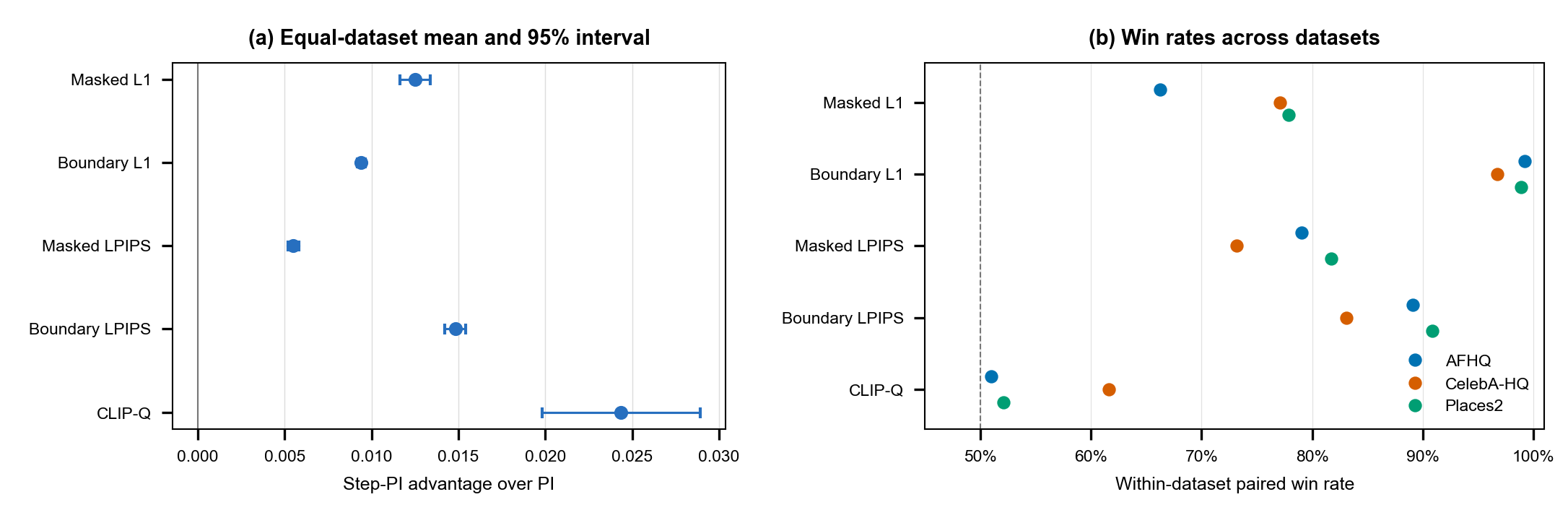}
  \caption{Paired Step-PI advantages over PI on Main35: (a)
  direction-normalized equal-dataset paired means with 95\%
  case-within-protocol bootstrap intervals; (b)
  within-dataset paired win rates across the same five metrics.  Positive
  values favor Step-PI; the dashed line marks 50\%; all 3,500 cases contribute.}
  \label{fig:native-paired-effects}
\end{figure*}

\subsection{Three-Domain Frozen-Backbone Inpainting}
\label{sec:main35-ladder}

Among the four matched variants, Step-PI is best in all 15 dataset--metric
cells across AFHQ, CelebA-HQ, and Places2, with performance improving
monotonically along
DDIM-Proj$\rightarrow$P-Guidance$\rightarrow$PI$\rightarrow$Step-PI
(Table~\ref{tab:main35-internal-absolute}).

\label{sec:configuration-locked-transfer}
These results cover all three datasets; AFHQ and Places2 are the two
transfer-only targets evaluated without retuning.

On an outcome-blind, protocol-stratified random subset of 175 cases, the
fixed Prompt-CLIP contrast
favors Correct over Empty, with the 95\% CI excluding zero.  This paired
intervention demonstrates that the complete frozen pipeline responds to prompt
presence under otherwise fixed inference conditions.  Complete endpoint
definitions, three-domain estimates, and strict-random visual examples are
provided in Supplementary Appendices~C and~D.

Figure~\ref{fig:cross-domain-teaser} shows fixed-configuration completions
reused from an outcome-independent random sample across all three datasets; the
full random construction-ladder grid is provided in Supplementary Appendix~D.

\subsection{Matched Mechanism Evidence}
\label{sec:main35-results}

Matched P-Guidance$\rightarrow$PI changes only cross-step persistent state,
whereas PI$\rightarrow$Step-PI changes only the complete predefined release
schedule.  With objectives, weights, and paired inputs fixed and unknown-region
ground truth unavailable, they isolate state and schedule rather than loss
selection or retuning; LPIPS and CLIP-Q remain evaluation-only.  The same
subset probes conditional one-field sensitivity.

PI improves P-Guidance in all 15 Main35 aggregates; all five stratified-bootstrap
intervals exclude zero and all five pooled per-case win rates exceed 50\%.
Step-PI likewise improves PI in all 15 aggregates, with all five intervals excluding zero.
Figure~\ref{fig:native-paired-effects} summarizes its five-metric release
effect; complete paired and win-rate tables are provided in Supplementary Appendix~C.

On the same 175-case subset, uniformizing any field worsens all five metrics.
Because the other three remain and $q_k$ and $r_I$ are serial scalings, this
supports conditional sensitivity within the implemented schedule---not
independent effective degrees of freedom, necessity, minimality, optimality, or
superiority to lower-dimensional or shared alternatives; details are provided in Supplementary Appendix~C.

\paragraph{Mechanism robustness across inference seeds.}
\label{sec:capability-results}
Across five paired inference-seed sweeps over the same outcome-blind, balanced
175-case stratified random subset spanning all 35 protocols (875 pairs), PI
outperforms P-Guidance on all five fixed evaluation metrics when persistent state is the
only component changed.  Across all five metrics, the hierarchical five-seed
effects consistently favor PI, with every 95\% CI excluding zero and every
Holm-corrected test remaining significant; full estimates are provided in
Supplementary Appendix~C.  On the same cohort, Step-PI outperforms PI on every
metric in every sweep when the release schedule is the only component changed;
pooled 95\% CIs exclude zero for all five metrics.  Full estimates, protocol
details, and reproducibility scope are provided in Supplementary Appendices~C
and~E.

\subsection{External Positioning and Computational Cost}
\label{sec:external-positioning}

\begin{table*}[t]
\centering
\small
\setlength{\tabcolsep}{0.41em}
\renewcommand{\arraystretch}{1.08}
\begin{tabular}{@{}llccccc@{}}
\toprule
Training type & Method & Masked L1 $\downarrow$ & Boundary L1 $\downarrow$ & Masked LPIPS $\downarrow$ & Boundary LPIPS $\downarrow$ & CLIP-Q $\uparrow$ \\
\midrule
\multirow{3}{*}{\shortstack[l]{\textit{Training-free on}\\\textit{vanilla SD1.5}}} & Step-PI & \textbf{0.1526} & \textbf{0.0291} & \textbf{0.1123} & \textbf{0.1029} & \textbf{0.5864} \\
 & LanPaint & 0.1608 & 0.0438 & 0.1323 & 0.1464 & 0.5495 \\
 & PILOT & 0.1596 & 0.0433 & 0.1238 & 0.1328 & 0.5684 \\
\midrule
\multirow{3}{*}{\shortstack[l]{\textit{Inpainting-trained}\\\textit{references}}} & SD-Inpaint & \underline{0.1198} & 0.0233 & \underline{0.0881} & 0.0733 & \underline{0.6508} \\
 & BrushNet & 0.1371 & 0.0292 & 0.1038 & 0.1010 & 0.6391 \\
 & PixelHacker & 0.1458 & \underline{0.0209} & 0.0993 & \underline{0.0713} & 0.6213 \\
\bottomrule
\end{tabular}
\caption{External positioning on Main35 by resource regime.  Bold marks the best result among training-free methods using vanilla SD1.5, and underlining marks the best overall result.  PixelHacker is additionally adapted to Places2; native routes differ, so the values are neither intervention-matched nor compute-normalized.}
\label{tab:external-equal-dataset-macro}
\end{table*}

\begin{figure*}[t]
  \centering
  \includegraphics[width=\textwidth]{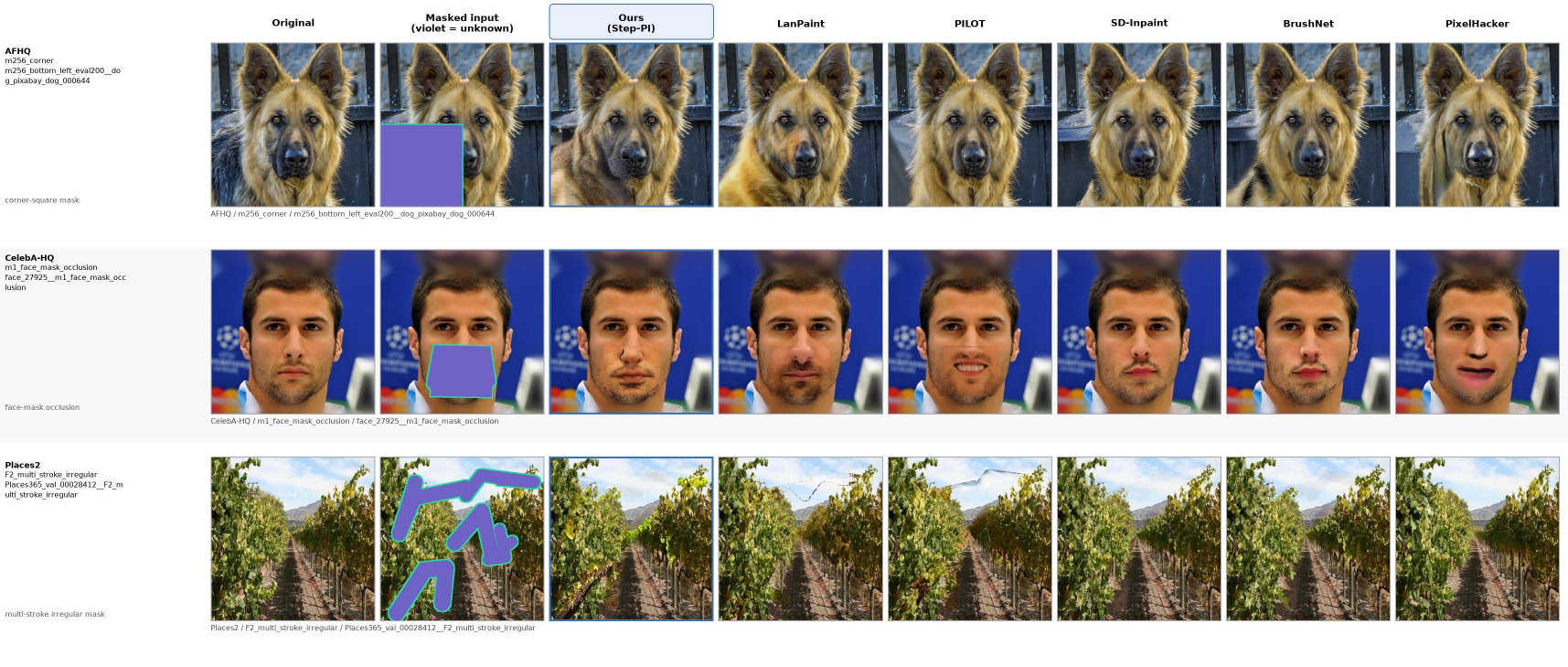}
  \caption{Native-route external comparison on three randomly drawn cases
  shared with the supplementary internal grid.  The draw did not depend on
  method outputs; all methods receive the same image/mask, and LanPaint and
  PILOT inherit the cases without reselection.}
  \label{fig:external-baseline-grid}
\end{figure*}

\paragraph{Training-free external positioning.}
LanPaint and PILOT, the closest evaluated vanilla-SD1.5 training-free
references, were not retuned on Main35: LanPaint uses its official v1.5.5
ComfyUI default configuration fixed before Main35, whereas PILOT uses a fixed
DDIM sampling configuration with periodic latent optimization.  Under these
audited native routes, Step-PI records better values on all
five equal-dataset macro metrics and on 13/15 and 12/15 cells, respectively
(Table~\ref{tab:external-equal-dataset-macro}).  This is a descriptive
comparison among training-free methods built on the same pretrained SD1.5
foundation. Because the methods use their native inference procedures and are
not matched for computational budget, the results should not be interpreted as
isolating algorithmic effects or establishing a comprehensive SOTA ranking;
full per-dataset results and an 18-case random atlas are provided in
Supplementary Appendices~D and~E.

\paragraph{Comparison with inpainting-trained systems.}
The inpainting-trained references encode substantial offline specialization:
SD-Inpaint uses 440k inpainting steps at $512\!\times\!512$
\citep{stableDiffusionInpaintingModelCard}; BrushNet reports 430k steps on eight
V100 GPUs (about three days) \citep{ju2024brushnet}; and PixelHacker uses 200k
iterations on 14M image--mask pairs plus a 120k-step Places2 fine-tune on 1.8M
images, both with 12 L40S GPUs \citep{xu2025pixelhacker}.  Step-PI instead
freezes vanilla SD1.5 and performs no inpainting- or dataset-specific weight
training.  The trained lane nevertheless retains higher absolute
scores---SD-Inpaint leads Masked L1, Masked LPIPS, and CLIP-Q, while
PixelHacker leads both boundary metrics---but obtains them with offline
optimization absent from Step-PI.  Even without such weight training, Step-PI
is numerically close to trained-reference results on some metrics
(Table~\ref{tab:external-equal-dataset-macro}).  Because the published budgets
differ in data, hardware, and batch size, they provide resource context rather
than a normalized quality--cost comparison.

\paragraph{Computational cost.}
A 35-case A100 audit measures Step-PI at 30.48~s/image and 15.55~GiB: no resolved
runtime difference from PI, 9.94$\times$ DDIM-Proj, and slower than PILOT/LanPaint
under descriptive native routes. Two feedback gradients per step without parameter updates trade
offline specialization for costly online control---not inference efficiency or
quality--compute superiority; full details are provided in Supplementary Appendix~E.

\FloatBarrier

\section{Conclusion}

\method repurposes a frozen SD1.5 text-to-image diffusion model for
training-free conditional inpainting across three evaluated domains through a
stateful boundary--interior latent controller with a predefined release
schedule, without weight updates or learned inpainting-specific conditioning
channels.  Developed on disjoint CelebA-HQ pilots, it transfers unchanged to
AFHQ and Places2.  Matched state-only and release-only controls improve all
15 Main35 dataset--metric cells.  Descriptively, \method leads LanPaint and
PILOT on all five equal-dataset macro metrics and approaches selected
inpainting-trained references on some metrics.  Inpainting-trained systems
retain the absolute metric leads but rely on substantial inpainting-specific
offline optimization.  Although \method incurs substantial inference cost,
these results position test-time latent control as a complementary route for
repurposing a frozen generic text-to-image prior.

\FloatBarrier

{
  \small
  \bibliographystyle{ieeenat_fullname}
  \bibliography{references}
}

\clearpage
\twocolumn[
  \begin{center}
    {\LARGE\bfseries Supplementary Material}\par
    \vspace{8pt}
    {\Large\bfseries Training-Free Inpainting Across Domains}\par
    {\Large\bfseries with a Frozen Text-to-Image Diffusion Model}
  \end{center}
  \vspace{8pt}
  \begin{minipage}{\textwidth}
  This supplement follows the main paper's evidence chain.  Appendix~A gives the
  implementation details, effective controller equations, predefined release schedule,
  and effective persistent-state characterization; Appendix~B records shared evaluation
  and reproduction protocols; Appendix~C pairs matched mechanism and text-only
  protocols with five-metric and prompt-alignment evidence; Appendix~D provides
  qualitative evidence and stress-test diagnostics; and Appendix~E reports
  computational costs, external positioning, and claim boundaries.
  \end{minipage}
  \vspace{12pt}
]

\appendix
\setcounter{tocdepth}{2}
\tableofcontents
\clearpage

\section{Method Implementation Details and State Characterization}
\label{sec:supp-reproducibility}
\label{sec:supp-method-details}

This appendix gives the executable method definition behind the compact
main-paper presentation.  We separate the projection and frozen transition,
the two feedback objectives, the effective persistent-state and action
equations, the predefined release schedule, and the bounded-state property.
Shared evaluation settings and the configuration-development boundary remain
in Appendix~B.

\subsection{Leakage-Free Projection and Frozen Transition}
\label{sec:supp-projection-transition}

Let $y\in[0,1]^{3\times H\times W}$ be the observed RGB image, let
$y^{\mathrm{enc}}=2y-1$ denote its VAE-encoder-range representation, and let
$M=1$ denote the unknown region and $K=1-M$ the visible region.  For odd $h$, let
$\operatorname{Box}_h$ denote channelwise stride-one average pooling with
zero padding of width $(h-1)/2$ and the padded samples included in the
denominator.  The same operator is used in the numerator and support map, so
their ratio is the average of available visible samples.  The context fill is
\begin{align}
A^{\mathrm{ctx}}&=\operatorname{Box}_{65}(K),\nonumber\\
F^{\mathrm{ctx}}&=\mathbf 1[A^{\mathrm{ctx}}>10^{-5}]\odot
\frac{\operatorname{Box}_{65}(K\odot y^{\mathrm{enc}})}
{\max(A^{\mathrm{ctx}},10^{-5})},\nonumber\\
y^{\mathrm{ctx}}&=K\odot y^{\mathrm{enc}}+M\odot F^{\mathrm{ctx}},\nonumber\\
z_0^{\mathrm{known}}&=s_{\mathrm{VAE}}\,
\mathbb E[q_{\mathrm{VAE}}(z\mid y^{\mathrm{ctx}})].
\label{eq:supp-context-filled-image}
\end{align}
Locations without visible support receive zero, and the posterior mean is used
rather than a sampled latent.  Because every occurrence of the image inside
the fill is multiplied by $K$, hidden RGB values do not enter
$z_0^{\mathrm{known}}$.

The latent mask is
$M_z=\mathbf 1[\operatorname{Resize}_{\mathrm{nearest}}(M)>0.5]$, so $M_z=1$
retains the unknown region.  One fixed noise tensor
$\epsilon^{\mathrm{known}}$ defines the complete known-region trajectory:
\begin{align}
q_t^{\mathrm{known}}
  &=\sqrt{\bar\alpha_t}\,z_0^{\mathrm{known}}
    +\sqrt{1-\bar\alpha_t}\,\epsilon^{\mathrm{known}},
\label{eq:supp-known-trajectory}\\
\Pi_t(z)&=(1-M_z)\odot q_t^{\mathrm{known}}+M_z\odot z.
\label{eq:supp-known-projection}
\end{align}
Thus visible coordinates follow one fixed forward-noise trajectory rather than
receiving fresh noise at each step.  Projection is applied before the frozen
U-Net transition and again after controller actuation.

Let $c$ be the fixed user-provided text condition, $c_\emptyset$ the empty
condition, and $w_{\mathrm{cfg}}$ the classifier-free-guidance scale.  At step
$k$, the frozen transition computes
\begin{align}
\widehat\epsilon^k
&=\epsilon_\theta(\widetilde z^k,t_k,c_\emptyset)
+w_{\mathrm{cfg}}\!\left[
\epsilon_\theta(\widetilde z^k,t_k,c)
-\epsilon_\theta(\widetilde z^k,t_k,c_\emptyset)\right],\nonumber\\
\widehat z_0^k
&=\frac{\widetilde z^k-
\sqrt{1-\bar\alpha_{t_k}}\,\widehat\epsilon^k}
{\sqrt{\bar\alpha_{t_k}}},\nonumber\\
z_{\mathrm{base}}^{k-1}
&=\sqrt{\bar\alpha_{t_{k-1}}}\,\widehat z_0^k
+\sqrt{1-\bar\alpha_{t_{k-1}}}\,\widehat\epsilon^k.
\label{eq:supp-frozen-transition}
\end{align}
This is deterministic DDIM with $\eta=0$ under the fixed scheduler
convention in Appendix~B.  Image-space feedback uses the unclipped FP32 decode
$\widehat x_0^k=\tfrac12[\operatorname{VAEdec}(
\widehat z_0^k/s_{\mathrm{VAE}})+1]$; neither the U-Net nor VAE weights are
updated.  After the terminal latent projection, the submitted image is
constructed as
\begin{align}
x^{\mathrm{raw}}
&=\operatorname{Clamp}_{[0,1]}\!\left(
\tfrac12[\operatorname{VAEdec}(z^0/s_{\mathrm{VAE}})+1]\right),\nonumber\\
x^{\mathrm{out}}&=K\odot y+M\odot x^{\mathrm{raw}}.
\label{eq:supp-final-composite}
\end{align}
Thus latent projection reimposes the known-region trajectory during sampling,
whereas the final RGB composite makes every known output pixel exactly equal
to the observed input; hidden-region reference pixels are not used.

\subsection{Boundary--Interior Feedback Objectives}
\label{sec:supp-feedback-objectives}

At reverse step $k$, set $\widehat x=\widehat x_0^k$.  Let $\Omega$ be the
pixel lattice, $C_{\mathrm{rgb}}=3$, and $\varepsilon=10^{-8}$.  For a
nonnegative scalar mask $W$, define
\begin{equation}
\langle U\rangle_W=
\frac{\sum_{p\in\Omega}\sum_{c=1}^{C_{\mathrm{rgb}}}W_pU_{p,c}}
{C_{\mathrm{rgb}}\sum_{p\in\Omega}W_p+\varepsilon}.
\label{eq:supp-weighted-average}
\end{equation}
Let $\Delta=\{(1,0),(-1,0),(0,1),(0,-1)\}$, with
$\nabla_\delta U_p=U_{p+\delta}-U_p$.  For each $\delta\in\Delta$, let
$E_\delta=\{p:K_pM_{p+\delta}=1\}$ collect visible pixels adjacent to the
unknown region.  The visible-context and seam-pair terms are
\begin{align}
\mathcal L_{\mathrm{known}}
&=\langle|\widehat x-y|\rangle_K,
\label{eq:supp-loss-known}\\
\mathcal L_{\mathrm{pair}}
&=\frac{\sum_{\delta\in\Delta}\sum_{p\in E_\delta}
\|\widehat x_{p+\delta}-y_p\|_1}
{C_{\mathrm{rgb}}\sum_{\delta\in\Delta}|E_\delta|+\varepsilon}.
\label{eq:supp-loss-pair}
\end{align}

Let $\operatorname{Dilate}_r(W)$ be same-size binary dilation by a square
structuring element of Chebyshev radius $r$, with the image exterior treated
as zero.  Define $B_r^{\mathrm{in}}=M\odot\operatorname{Dilate}_r(K)$ and
$B_r^{\mathrm{out}}=K\odot\operatorname{Dilate}_r(M)$.  Let
$\operatorname{TV}(U)$ be the mean horizontal absolute difference plus the
mean vertical absolute difference over the full tensor, and let
$G_\delta=\{p:K_pK_{p+\delta}M_{p+2\delta}M_{p+3\delta}=1\}$.  The remaining
boundary terms are
\begin{align}
\mathcal L_{\mathrm{TV}}
&=\operatorname{TV}(B_8^{\mathrm{in}}\odot\widehat x),
\label{eq:supp-loss-tv}\\
\mathcal L_{\mathrm{boundary\text{-}grad}}
&=\frac{\sum_{\delta\in\Delta}\sum_{p\in G_\delta}
\|\nabla_\delta\widehat x_{p+2\delta}-\nabla_\delta y_p\|_1}
{C_{\mathrm{rgb}}\sum_{\delta\in\Delta}|G_\delta|+\varepsilon}.
\label{eq:supp-loss-boundary-gradient}
\end{align}
Equation~\eqref{eq:supp-loss-tv} is the executed \emph{support-masked image
variation term}: because TV is evaluated after zeroing pixels outside the inner
ring, it includes both differences within the ring and differences across the
ring-support edge.  It should not be interpreted as a masked-adjacency TV that
uses only pairs lying wholly inside the ring, nor does this implementation-level
term isolate a seam-specific benefit.  The RGB-gradient term compares
two generated pixels just inside the mask with two observed pixels just
outside it.  Empty edge sets contribute zero.

The interior objective uses radii $\mathcal R=\{16,32,64\}$.  For a positive
integer $u$, let $\operatorname{odd}(u)$ be the smallest odd integer no smaller
than $u$.  For $r\in\mathcal R$, define
\begin{align}
A_r&=\operatorname{Box}_{2r+1}(K),\qquad
V_r=\mathbf 1[A_r>10^{-4}],\nonumber\\
T_r^{\mathrm{ctx}}&=
\frac{\operatorname{Box}_{2r+1}(K\odot y)}{A_r+\varepsilon},
\qquad
G_r=\operatorname{Box}_{\operatorname{odd}(r+1)}(\widehat x).
\label{eq:supp-context-fields}
\end{align}
The target $T_r^{\mathrm{ctx}}$ is detached when differentiated.  A
normalized distance-to-known weight is constructed by at most 24 one-pixel
erosions: $R^{(0)}=M$,
$R^{(j+1)}=R^{(j)}\odot[1-\operatorname{Dilate}_1(1-R^{(j)})]$, and
\begin{equation}
D=M\odot\frac{1}{24}\sum_{j=0}^{23}R^{(j)}.
\label{eq:supp-interior-distance}
\end{equation}
The multiscale continuation terms are
\begin{align}
\mathcal L_{\mathrm{lowfreq}}
&=\frac{1}{|\mathcal R|}\sum_{r\in\mathcal R}
\langle|G_r-T_r^{\mathrm{ctx}}|\rangle_{B_r^{\mathrm{in}}\odot V_r},
\label{eq:supp-loss-lowfreq}\\
\mathcal L_{\mathrm{interior}}
&=\frac{1}{|\mathcal R|}\sum_{r\in\mathcal R}
\langle|G_r-T_r^{\mathrm{ctx}}|\rangle_{M\odot D\odot V_r}.
\label{eq:supp-loss-interior}
\end{align}

For a mask $W$, let $\mu_W(U)$ and $\sigma_W(U)$ be per-channel weighted
spatial moments.  The radius-32 ring term is
\begin{equation}
\begin{aligned}
\mathcal L_{\mathrm{ring}}
&=\frac{1}{C_{\mathrm{rgb}}}
\|\mu_{B_{32}^{\mathrm{in}}}(\widehat x)
-\mu_{B_{32}^{\mathrm{out}}}(y)\|_1\\
&\quad+\frac{1}{C_{\mathrm{rgb}}}
\|\sigma_{B_{32}^{\mathrm{in}}}(\widehat x)
-\sigma_{B_{32}^{\mathrm{out}}}(y)\|_1.
\end{aligned}
\label{eq:supp-loss-ring}
\end{equation}
Define the $9\!\times\!9$ high-pass fields
\begin{align}
H_g&=\widehat x-\operatorname{Box}_{9}(\widehat x),\nonumber\\
H_k&=K\odot\left[y-
\frac{\operatorname{Box}_{9}(K\odot y)}
{\max(\operatorname{Box}_{9}(K),\varepsilon)}\right],
\end{align}
and $e(H)_p=C_{\mathrm{rgb}}^{-1}\sum_c|H_{p,c}|$.  The frequency term is
\begin{equation}
\mathcal L_{\mathrm{frequency}}=
\left|
\frac{\sum_p(M\odot D)_pe(H_g)_p}{\sum_p(M\odot D)_p+\varepsilon}
-\frac{\sum_pB_{32,p}^{\mathrm{out}}e(H_k)_p}
{\sum_pB_{32,p}^{\mathrm{out}}+\varepsilon}
\right|.
\label{eq:supp-loss-frequency}
\end{equation}

The boundary and interior objectives use the component order in the main paper
with fixed weights
\begin{equation}
\boldsymbol\lambda_B=(0.50,1.00,0.05,0.20),\qquad
\boldsymbol\lambda_I=(0.20,0.15,0.02,0.05).
\label{eq:supp-feedback-weights}
\end{equation}
With $D_z=\operatorname{Resize}_{\mathrm{area}}(D)$, their effective feedback
directions are
\begin{align}
g_B^k&=\operatorname{Normalize}_2
[-M_z\odot\nabla_{\widetilde z^k}\mathcal L_B],\nonumber\\
g_I^k&=\operatorname{Normalize}_2
[-D_z\odot\nabla_{\widetilde z^k}\mathcal L_I],
\label{eq:supp-normalized-feedback}
\end{align}
where $\operatorname{Normalize}_2(v)=v/(\lVert v\rVert_2+10^{-8})$ is
applied per sample over the full latent tensor.  These directions are
differentiated at the current projected latent, while the shaped action is
applied at the post-transition latent port.  Normalization supplies fixed
cross-objective and cross-timestep calibration rather than raw-magnitude
adaptivity; raw gradient norm is neither a stopping rule nor a convergence
certificate.  Full-tensor normalization fixes the total per-sample direction
norm rather than the per-active-coordinate RMS; this is an execution convention,
not a claim of per-pixel or mask-area-invariant actuation.  Appendix~A.5
verifies the equivalence between the frozen
implementation and these effective equations.

\subsection{Effective Persistent State and Bounded Action}
\label{sec:supp-controller-analysis}
\label{sec:supp-pi-states}

Let $\operatorname{InnerBoundary}(W,d)=
W\odot\operatorname{Dilate}_d(1-W)$.  The latent inner-boundary band,
remaining deep interior, boundary-gradient weighting, and effective spatial
retention are
\begin{align}
S_z&=\operatorname{InnerBoundary}(M_z,d_B),
\label{eq:supp-boundary-mask}\\
I_z&=M_z-S_z,\qquad W_B=S_z+\omega_B I_z,
\label{eq:supp-boundary-weight}\\
\beta_B&=\mu_{B,S}S_z+\mu_{B,I}I_z.
\label{eq:supp-boundary-retention}
\end{align}
Both states are initialized to zero.  Boundary-state formation is confined to
$M_z$, whereas the interior state is driven by the area-resized depth field
$D_z$; the final combined action is masked by $M_z$ before actuation.  The
effective recurrences are
\begin{align}
\xi_B^k&=\operatorname{ClipNorm}_{\nu_B}\!\left[
\beta_B\odot\xi_B^{k+1}+(1-\beta_B)\odot(W_B\odot g_B^k)\right],
\label{eq:supp-boundary-memory}\\
\xi_I^k&=\gamma_I\xi_I^{k+1}+(1-\gamma_I)g_I^k.
\label{eq:supp-interior-memory}
\end{align}
Here $\operatorname{ClipNorm}_{\nu}(v)$ is the per-sample Euclidean
projection of the full latent tensor onto the radius-$\nu$ ball.  Spatially
varying $\beta_B$ retains boundary and interior history at different rates,
while $W_B$ attenuates current deep-interior feedback.  The interior direction
is already depth-weighted by $D_z$, so its retention is scalar.  The evaluated
state values are
\begin{equation}
\begin{aligned}
d_B&=2,\quad \omega_B=0.15,\quad \nu_B=1,\\
(\mu_{B,S},\mu_{B,I})&=(0.95,0.70),\quad \gamma_I=0.90.
\end{aligned}
\label{eq:supp-state-parameters}
\end{equation}
Both states update at every reverse step for PI and Step-PI.  They represent
EMA-like, exponentially discounted histories of normalized feedback directions.
The PI-structured label denotes the explicit current-versus-persistent-history
construction rather than textbook integration of raw stationary error; later
feedback can counteract inconsistent history.  The fixed VAE lattice makes the
elementwise recurrence well defined but does not assert timestep-invariant
direction geometry or cross-timestep gradient transport.

For $j\in\{B,I\}$, let the release coefficients
$(\alpha_{P,k}^j,\alpha_{I,k}^j,\alpha_{O,k}^j)$ act on the proportional,
integral, and fused output terms.  The complete bounded action is
\begin{equation}
\begin{aligned}
(\alpha_{P,k}^B,\alpha_{I,k}^B,\alpha_{O,k}^B)
  &=(1,\rho_{B,k},1),\\
(\alpha_{P,k}^I,\alpha_{I,k}^I,\alpha_{O,k}^I)
  &=(\rho_{P,k},\rho_{P,k}\rho_{H,k},\rho_{O,k}),\\
P_j^k&=\alpha_{P,k}^jK_P^jg_j^k,\\
I_j^k&=\operatorname{CapNorm}\!\left(
\alpha_{I,k}^jK_I^j\xi_j^k,c_j\lVert P_j^k\rVert_2\right),\\
u_j^k&=\alpha_{O,k}^j(P_j^k+I_j^k),\\
a_k&=\operatorname{Clamp}_{[-a_{\max},a_{\max}]}
\!\left[M_z\odot(u_B^k+u_I^k)\right].
\end{aligned}
\label{eq:supp-bounded-action}
\end{equation}
For each sample and full latent tensor,
\begin{equation}
\operatorname{CapNorm}(v,r)=
v\min\!\left(1,\frac{r}{\lVert v\rVert_2+10^{-8}}\right).
\label{eq:supp-capnorm}
\end{equation}
Thus the integral term is capped relative to its proportional counterpart
before the shared final elementwise clamp.  The evaluated gains and
potentially nonidentity action safeguards are
\begin{equation}
\begin{aligned}
(K_P^B,K_I^B)&=(0.08,0.40),\\
(K_P^I,K_I^I)&=(0.20,0.18),\\
c_B=c_I&=1.5,\qquad a_{\max}=0.12.
\end{aligned}
\label{eq:supp-action-parameters}
\end{equation}
State formation is schedule independent: release controls immediate actuation
but never skips either feedback gradient or either state update.

\subsection{Predefined Release Schedule}
\label{sec:supp-release-contract}
\label{sec:supp-release-schedules}

For executed steps $k\in\{N,\ldots,1\}$, let
$p_k=(N-k)/(N-1)\in[0,1]$ be normalized reverse-step progress and
$s(x)=x^2(3-2x)$ the cubic smoothstep between knots.  Step-PI uses one
predefined release schedule parameterized by four fields.  Boundary-memory and final-interior output use the
scheduler coefficient,
\begin{align}
r_B(t_k)&=r_B^{\min}+(r_B^{\max}-r_B^{\min})\bar\alpha_{t_k},\nonumber\\
r_I(t_k)&=\max(\bar\alpha_{t_k},r_I^{\min}),
\label{eq:supp-alpha-release}
\end{align}
whereas common interior release uses
\begin{equation}
q_k=\begin{cases}
q_0,&p_k\le a_1,\\
q_0+(q_m-q_0)s\!\left(\frac{p_k-a_1}{a_2-a_1}\right),&a_1<p_k<a_2,\\
q_m,&a_2\le p_k\le a_3,\\
q_m-(q_m-q_f)s\!\left(\frac{p_k-a_3}{a_4-a_3}\right),&a_3<p_k<a_4,\\
q_f,&p_k\ge a_4,
\end{cases}
\label{eq:supp-interior-midband}
\end{equation}
and the additional interior-memory release is
\begin{equation}
h_k=\begin{cases}
1,&p_k\le b_1,\\
1-s\!\left(\frac{p_k-b_1}{b_2-b_1}\right),&b_1<p_k<b_2,\\
0,&p_k\ge b_2.
\end{cases}
\label{eq:supp-interior-i-release}
\end{equation}
The predefined release constants are
\begin{equation}
\begin{aligned}
(r_B^{\min},r_B^{\max},r_I^{\min})&=(0.50,1.00,0.05),\\
(q_0,q_m,q_f)&=(0.25,1.00,0.15),\\
(a_1,a_2,a_3,a_4)&=(0.10,0.30,0.70,0.90),\\
(b_1,b_2)&=(0.55,0.80).
\end{aligned}
\label{eq:supp-release-parameters}
\end{equation}
The two parameterizations serve distinct implementation roles:
$\bar\alpha_{t_k}$ is the scheduler's cumulative signal coefficient and
monotonically tracks denoising progress on the fixed route, while $p_k$
fixes phase locations along the executed 50-step trajectory.  Both are
deterministic monotone coordinates of that same trajectory, not independent
state observations; robustness to a different
scheduler or step count was not evaluated.  The four fields are hand specified
and frozen, not learned or inferred from realized posterior uncertainty.

The release vector is
\begin{equation}
\rho_k^{\mathrm{PI}}=(1,1,1,1),\qquad
\rho_k^{\mathrm{Step\text{-}PI}}=(r_B(t_k),q_k,h_k,r_I(t_k)).
\label{eq:supp-release-vector}
\end{equation}
PI and Step-PI share the model, scheduler, inputs, randomness, projection,
objectives, both gradients, gains, effective state, norm caps, final clamp, and
action placement; only these four release fields differ.  In the interior
branch, $q_k$ scales both $P_I^k$ and the pre-CapNorm integral input, while its
CapNorm radius scales through $\lVert P_I^k\rVert_2$; $r_I(t_k)$ then scales
their fused output.  Thus, up to the $10^{-8}$ stabilizer and downstream
safeguards, the effective interior output scale depends approximately on
$q_kr_I(t_k)$.  The four fields record the executed configuration but are not
four independent effective degrees of freedom.  Because release is applied
before the shared $\operatorname{CapNorm}$ and final clamp, a field change need
not induce a linear final-action change.  The main comparisons test the complete
predefined schedule, and Appendix~C's one-field-uniformization audit probes
conditional sensitivity by setting only the selected field to one while the
other three retain their predefined trajectories.  For the selected field,
this gives uniform release; it neither disables a channel nor gates
computation.  Neither comparison establishes minimality, global
schedule-shape optimality, or dominance over every lower-dimensional
alternative.
The evaluated PI-to-Step-PI contrast identifies the effect of the complete
predefined schedule as implemented; it does not separately identify temporal
variation from the associated overall actuation profile or from a gain-matched
constant-release alternative.

\subsection{Effective Persistent-State Recurrence}
\label{sec:supp-persistent-feedback-analysis}

This subsection proves the main-paper property for the effective recurrence
and then verifies why three safeguards retained in the implementation reduce
to the identity under the evaluated configuration.  Put
\begin{equation*}
(m_B,w_B)=(\beta_B,W_B),\qquad (m_I,w_I)=(\gamma_I,1),
\end{equation*}
let $d_j^k=w_j\odot g_j^k$, and define
\begin{equation*}
\mathcal C_B=\{\xi:\lVert\xi\rVert_2\le\nu_B\},\qquad
\mathcal C_I=\mathbb R^{4\times64\times64}.
\end{equation*}
Write $\mathcal P_j$ for Euclidean projection onto $\mathcal C_j$; hence
$\mathcal P_B=\operatorname{ClipNorm}_{\nu_B}$ and
$\mathcal P_I=\operatorname{Id}$.  For
\begin{equation*}
\zeta_j^k=m_j\odot\xi_j^{k+1}+(1-m_j)\odot d_j^k,
\end{equation*}
the effective update is $\xi_j^k=\mathcal P_j(\zeta_j^k)$.

\begin{proof}[Proof of the effective-state proposition]
Consider
\begin{align*}
Q_j^k(\xi)={}&
\frac{1}{2}\lVert\sqrt{m_j}\odot(\xi-\xi_j^{k+1})\rVert_2^2\\
&+\frac{1}{2}\lVert\sqrt{1-m_j}\odot(\xi-d_j^k)\rVert_2^2.
\end{align*}
The elementwise weights sum to one, so
$\nabla Q_j^k(\xi)=\xi-\zeta_j^k$ and the unconstrained minimizer is
$\zeta_j^k$.
The unique minimizer over $\mathcal C_j$ is therefore its Euclidean
projection $\mathcal P_j(\zeta_j^k)=\xi_j^k$.

For the boundary state, $W_B\in[0,1]$ and normalized feedback gives
$\lVert d_B^k\rVert_2\le1=\nu_B$, so $d_B^k\in\mathcal C_B$; the interior
set is unconstrained and contains $d_I^k$.  Nonexpansiveness of Euclidean
projection, with the identity as the interior special case, gives
\begin{align*}
\lVert\xi_j^k-d_j^k\rVert_2
&=\lVert\mathcal P_j(\zeta_j^k)-\mathcal P_j(d_j^k)\rVert_2\\
&\le\lVert \zeta_j^k-d_j^k\rVert_2\\
&=\lVert m_j\odot(\xi_j^{k+1}-d_j^k)\rVert_2\\
&\le\lVert m_j\rVert_\infty
\lVert\xi_j^{k+1}-d_j^k\rVert_2.
\end{align*}
\end{proof}

The result characterizes a per-step bounded compromise toward the current
weighted direction.  Because $d_j^k$ changes across reverse steps, it does not
imply trajectory-level contraction, stability, task benefit, release-schedule
effectiveness, or descent of the evaluated Step-PI trajectory.

\paragraph{Equivalence Between the Implementation and Effective Recurrence.}
Under the evaluated configuration, three retained safeguards do not alter the
effective recurrence.  Normalized feedback has $\ell_2$ norm below one, so its
subsequent elementwise clamp to $[-1,1]$ is the identity.  Boundary state and
weighted feedback have zero support in the known region, which remains zero at
every step.  Finally, the interior update is a convex combination of vectors in
the unit $\ell_2$ ball, initialized at zero, so its radius-one norm guard is also
the identity.  These are algebraic consequences of the stated configuration,
not empirical activation-rate claims.  The boundary norm projection,
integral-to-proportional $\operatorname{CapNorm}$, and final elementwise action
clamp remain potentially nonidentity safeguards.

\section{Shared Evaluation and Reproduction Protocol}
\label{sec:supp-evaluation-protocols}

This section records the evaluation choices shared across studies separately
from the method specification in Appendix~A.  It defines the fixed inputs and
internal inference stack, metric estimands and statistical scope, controller
development and transfer boundary, and audited native routes for selected
external methods, with full execution details for the two actively reproduced
training-free baselines.  Study-specific protocols remain adjacent to their
results in Appendix~C.

\paragraph{Reproduction package.}
All reported experiments use fixed configurations, case lists, random seeds,
model versions, and documented evaluation procedures.  Upon publication, we
will release the Step-PI implementation, experiment configurations, case
lists, evaluation scripts, and table/figure generation code.  Public datasets
and third-party checkpoints will be linked through their official sources and
licenses.  Reproduction refers to rerunning the specified experimental
protocols and statistical analyses; independent GPU executions are not
required to produce bitwise-identical image files.

\subsection{Frozen Inputs, Text Conditions, and Internal Inference Stack}
\label{sec:supp-frozen-inputs}

All internal training-free comparisons consume fixed $512\!\times\!512$ RGB
images.  CelebA-HQ uses its native 512-pixel images, while the AFHQ and Places2
case lists identify the 512-pixel inputs actually consumed by inference.
Masks use $M=1$ (or 255) for the unknown region and $M=0$ for the known region;
any size alignment uses nearest-neighbor interpolation, with values strictly
greater than 127 treated as unknown.  Text conditions are fixed before
evaluation.  AFHQ uses \emph{a realistic photo of a cat} for cat cases and
\emph{a realistic photo of a dog} for dog cases; CelebA-HQ uses \emph{a
realistic portrait photo of a person} for every case.  Places2 uses per-case
conditions derived from the corresponding frozen Places365 class label, for
example \emph{a realistic indoor photo of a jacuzzi}, \emph{a realistic indoor
photo of a basketball court}, and \emph{a realistic photo of an auto factory}.
These deliberately simple class-level conditions avoid output-based prompt
engineering.  They support only the paper's claim of prompt-presence
responsiveness of the frozen end-to-end pipeline; they do not evaluate
fine-grained masked-content control, instruction following, or paraphrase
robustness.  Appendix~C separately tests that responsiveness under a matched
Correct-versus-Empty intervention.

Table~\ref{tab:supp-main35-registry-summary} summarizes the Main35 evaluation
cohort consumed by all internal comparisons.  Every protocol contains 100
image--mask cases.  The case list fixes the exact source, mask, case identity,
and $512\!\times\!512$ inference input; the summary below does not infer any
additional raw-image preprocessing beyond those fixed inputs.  The exact case
identities and deterministic subset definitions will accompany the released
case lists.

\begin{table*}[t]
  \centering
  \scriptsize
  \setlength{\tabcolsep}{0.45em}
  \renewcommand{\arraystretch}{1.08}
  \begin{tabularx}{\textwidth}{l p{0.25\textwidth} c p{0.31\textwidth} c}
    \toprule
    Dataset & Frozen evaluation source & Protocols / cases & Protocol families & Protocol-mean mask area \\
    \midrule
    AFHQ & AFHQ-v2 validation split & 13 / 1,300 & Geometric, positional, and irregular-mask protocols & 9.7--40.1\% \\
    CelebA-HQ & CelebA-HQ native 512-pixel validation inputs & 14 / 1,400 & Geometric, free-form, semantic-irregular, and face-adaptive protocols & 2.6--37.2\% \\
    Places2 & Fixed 800-image subset of the Places365 standard validation split & 8 / 800 & Geometric and free-form protocols & 6.2--37.5\% \\
    \bottomrule
  \end{tabularx}
  \caption{Compact Main35 evaluation-cohort summary.  Mask-area ranges are the minimum
  and maximum protocol means within each dataset, not per-case extrema.  All
  35 protocols are formal-evaluation-only and contribute equally within their
  dataset.}
  \label{tab:supp-main35-registry-summary}
\end{table*}

DDIM-Proj, P-Guidance, PI, and Step-PI use the same frozen Stable Diffusion
v1.5 revision \texttt{451f4fe1}, with the original VAE,
$512\!\times\!512$ inputs, and $4\!\times\!64\!\times\!64$ latents.  The diffusion
pipeline uses FP16 parameters, while feedback VAE decoding and feedback-objective calculations
execute in FP32; no LoRA or replacement VAE is loaded.  The VAE scaling factor
is $0.18215$.

The runtime scheduler is DDIM with 1,000 training states,
$(\beta_{\min},\beta_{\max})=(0.00085,0.012)$, scaled-linear betas,
epsilon prediction, leading timestep spacing, \texttt{steps\_offset=1},
\texttt{set\_alpha\_to\_one=false}, and clipping and thresholding disabled.
The evaluated route uses 50 descending states $(981,961,\ldots,21,1)$, $\eta=0$,
initialization scale 1, CFG 7.5, and an empty negative prompt.  A joint per-case
generator draws unknown-region initial noise first and fixed known-region noise
second.  Within each case, the four internal methods share the image, mask,
text condition, sample seed, both noise tensors, and all frozen transition
settings.  Appendix~C states the exact field identities for the two adjacent
field-identical comparisons.
Here, deterministic DDIM denotes the $\eta=0$ transition conditional on fixed
inputs and random draws; it does not imply bitwise-identical GPU arithmetic
across independent executions.

\subsection{Metrics, Estimands, and Conditional Statistical Scope}
\label{sec:supp-main35-relative-gains}

\paragraph{Metric geometry and direction.}
Masked L1 is mean RGB error over the unknown mask.  Composite-based Masked
LPIPS uses LPIPS-Alex \citep{zhang2018unreasonable} (package version 0.1.4)
on the full-frame prediction after observed known-region pixels are copied
exactly into the evaluated composite; RGB uint8 inputs are mapped to $[-1,1]$.
It is computed on the full-frame composite rather than by spatially
mask-weighting LPIPS feature maps.  Here, ``Masked'' refers to eliminating
known-region pixel errors through compositing, not to masking intermediate
LPIPS features; receptive fields may therefore include the unknown-region
boundary neighborhood.  Boundary L1 uses a
Chebyshev-radius-2 dilation/erosion XOR band and probes the narrow pixel seam.
Boundary LPIPS uses the same LPIPS-Alex implementation on an inner
unknown-region ring of radius 12 within a 32-pixel-padded crop, probing a wider
perceptual neighborhood.  The two boundary metrics are complementary probes
rather than interchangeable estimates of one geometry.

CLIP-Q is our CLIP-IQA-style frozen two-prompt quality proxy
\citep{radford2021clip,wang2023exploringclip}.  It uses
\nolinkurl{openai/clip-vit-base-patch32} at revision \texttt{3d74acf9}
on the unknown-mask bounding-box crop with 32-pixel padding and Gaussian-blur
radius 0.2.  We apply a two-way softmax to the logits for
\emph{a high quality natural realistic photo} and \emph{a low quality
unnatural distorted image}, and report the positive-prompt probability.  Lower
is better for the four error metrics and higher is better for this proxy.
Together the five metrics cover reconstruction, local seam behavior, and a
no-reference quality proxy; they do not establish human preference or
fine-grained semantic compliance.

\paragraph{Aggregate relative gains.}
For an ordered main-paper Table~1(b) comparison
$(w,c)\in\{(\mathrm{PI},\mathrm{P\text{-}Guidance}),
(\mathrm{Step\text{-}PI},\mathrm{PI})\}$, the target method $w$ is listed
first and $c$ is its comparator.  For dataset $d$, the gain for a
lower-is-better error metric $e$ is
\begin{equation}
g^{(e)}_{d,w,c}
=100\frac{e_{d,c}-e_{d,w}}{e_{d,c}},
\label{eq:supp-main35-error-gain}
\end{equation}
whereas for the higher-is-better CLIP-Q score $s$ it is
\begin{equation}
g^{(s)}_{d,w,c}
=100\frac{s_{d,w}-s_{d,c}}{s_{d,c}}.
\label{eq:supp-main35-score-gain}
\end{equation}
For metric $m$, the summary weights AFHQ, CelebA-HQ, and Places2 equally:
\begin{equation}
\bar g_{w,c,m}=\frac{1}{3}
\sum_{d\in\{\mathrm{AFHQ},\,\mathrm{CelebA\text{-}HQ},\,\mathrm{Places2}\}}
g_{d,w,c,m}.
\label{eq:supp-main35-equal-dataset-gain}
\end{equation}
Positive values favor the first-listed method $w$.  The main-paper Table~1(b)
reports these relative means and favorable dataset--metric cells for both
adjacent comparisons; its absolute panel retains all four internal methods and
every per-dataset value.

\paragraph{Paired raw effects.}
The paired intervals and win rates use native metric differences rather than
the percentage gains above.  For a case $i$, metric $m$, left-hand method $a$,
and right-hand method $b$, define
\begin{equation}
\Delta_{i,m}^{a\rightarrow b}=
\begin{cases}
v_{i,m}^{a}-v_{i,m}^{b}, & m\text{ is lower-is-better},\\
v_{i,m}^{b}-v_{i,m}^{a}, & m\text{ is higher-is-better}.
\end{cases}
\label{eq:supp-main35-paired-raw-effect}
\end{equation}
Thus positive raw effects favor the right-hand method.  Display-only scaling,
when used, does not convert these differences into relative percentages.

\paragraph{Pairing and statistical scope.}
The paired Main35 analyses contain matched observations for all 3,500
image--mask conditions.  Appendix~C states the exact field identities of the
P-Guidance-to-PI and PI-to-Step-PI comparisons alongside their results.  Here
we define only their shared aggregation and conditional inferential scope.

Cases are averaged within protocol, protocols equally within dataset, and
datasets equally across the three evaluated domains.  Paired intervals use
10,000 case-within-protocol stratified-bootstrap draws with seed 20260719,
resampling cases within protocol before averaging the fixed protocol and
dataset means.  Some AFHQ source images recur across mask protocols, whereas
CelebA-HQ and Places2 use source-unique cases.  The bootstrap conditions on
this fixed cross-protocol reuse structure and does not additionally cluster by
source identity across protocols.  The separate MS175 cohort in Appendix~C is
source-unique within each dataset.  Each Main35 case uses one fixed inference
seed, varying across cases.  The resulting
intervals quantify case variation conditional on the three evaluated datasets,
35 fixed mask protocols, and fixed per-case seeds; they do not estimate
uncertainty over unseen domains, unseen mask families, or new inference seeds.
Appendix~C therefore reports separate five-seed MS175 audits for both adjacent
mechanism contrasts: P-Guidance-to-PI for persistent state and PI-to-Step-PI
for the complete predefined release schedule.  These repeated-initialization audits do not
estimate training-seed uncertainty or generalization to unseen datasets, mask
families, backbones, or samplers.

\subsection{Configuration Development and Transfer Scope}
\label{sec:supp-parameter-development}

All controller choices were finalized using a fixed five-case CelebA-HQ
development pilot confirmed disjoint from Main35.  Development used a mixture of
visual assessment and numerical metrics.  No Main35 formal-evaluation output
from AFHQ, CelebA-HQ, or Places2 entered configuration selection.  The single set of gains,
objective weights, retention values, bounds, and release schedules reported in
Appendix~A was then held fixed for every formal run, with no dataset-specific
retuning.  This supports configuration-locked transfer without target-domain
retuning; it does not establish hyperparameter optimality or insensitivity to
the pilot set.

Here cross-domain transfer denotes unchanged reuse from CelebA-HQ development
to the evaluated AFHQ and Places2 targets within one fixed SD1.5, original VAE,
$512\!\times\!512$, 50-step DDIM stack.  The study does not evaluate transfer
across backbones, VAEs, resolutions, samplers, or step counts.

\newpage

\subsection{External Method Attributes and Native Routes}
\label{sec:supp-external-method-contracts}

\begin{table*}[t]
\centering
\scriptsize
\setlength{\tabcolsep}{0.30em}
\begin{tabularx}{\textwidth}{l p{0.13\textwidth} c c c c X}
\toprule
Method & Evaluated foundation & \shortstack{Inpainting-\\trained weights} &
\shortstack{Learned\\inpainting interface} &
\shortstack{Dataset-specific\\weight adaptation} &
\shortstack{Test-time\\backprop} & Native sampler and interface \\
\midrule
Step-PI & Vanilla SD1.5 & No & No & No & Yes & 50-step DDIM; projected image/mask and objective feedback; generic text cross-attention \\
LanPaint & Vanilla SD1.5 & No & No & No & No & 30-step Euler/Karras plus Langevin conditioning; image, mask, and text \\
PILOT & Vanilla SD1.5 & No & No & No & Yes & 100-step DDIM with periodic latent optimization; image, mask, and text \\
SD-Inpaint & SD1.5-Inpaint & Yes & Yes & No & No & 50-step DDIM; learned masked-image/mask/text inpainting interface \\
BrushNet & SD1.5 plus BrushNet & Yes & Yes & No & No & 50-step UniPC; learned dual branch for masked-image/mask features plus text \\
PixelHacker & Places2-finetuned model and VAE & Yes & Yes & \shortstack{Yes\\(Places2)} & No & 20-step DDIM; image/mask and learned LCG embeddings; no text prompt \\
\bottomrule
\end{tabularx}
\caption{Method-attribute audit for Step-PI and the five selected external
baselines.  All evaluated weights are frozen at inference.  Inpainting-trained
weights and dataset-specific weight adaptation are distinct: SD-Inpaint and BrushNet
use inpainting-trained components without evaluation-domain retuning, whereas the
evaluated PixelHacker checkpoint is Places2-finetuned and reused unchanged on
AFHQ and CelebA-HQ.  The external rows are descriptive complete-system routes,
not strict one-switch comparisons with Step-PI.}
\label{tab:supp-external-method-attributes}
\end{table*}

\paragraph{Version pins.}
LanPaint \citep{zheng2025lanpaint} uses official release 1.5.5 through ComfyUI
v0.26.0; PILOT \citep{pan2024coherent} uses its official repository at commit
\texttt{ea869f9e}.  Both use vanilla four-channel Stable Diffusion v1.5
revision \texttt{451f4fe1}.  We evaluate the official LanPaint v1.5.5 ComfyUI
default configuration fixed before Main35.  The reported results characterize
this reproducible native route and do not rank alternative LanPaint
configurations.  PILOT's native route was likewise frozen before execution.

\begin{table*}[t]
  \centering
  \scriptsize
  \setlength{\tabcolsep}{0.45em}
  \renewcommand{\arraystretch}{1.10}
  \begin{tabularx}{\textwidth}{p{0.075\textwidth} p{0.20\textwidth} p{0.30\textwidth} X}
    \toprule
    Method & Foundation and sampler & Native test-time route & Frozen selection policy \\
    \midrule
    LanPaint & Vanilla SD1.5; 30-step Euler/Karras; CFG 5; empty negative prompt & Five inner steps; $\lambda=16$, step size 0.2, $\beta=1$, friction 15, Image-First mode, early stop 1, threshold 0, patience 1 & Official v1.5.5 ComfyUI default configuration fixed before Main35; no Main35 tuning, search, retry, or best-of-$N$ selection \\
    PILOT & Vanilla SD1.5; 100-step DDIM; CFG 7.5; $\eta=0$; FP16 parameters & Every tenth diffusion step triggers ten latent-gradient operations (100 total); $\gamma=1$, learning rates 0.007/0.025, coefficient 150, momentum 0.7; math-SDP & Frozen native route; no Main35 tuning, search, retry, or best-of-$N$ selection \\
    \bottomrule
  \end{tabularx}
  \caption{Actively reproduced training-free external routes.  These are
  complete-system native configurations, not matched controller interventions.}
  \label{tab:supp-training-free-native-routes}
\end{table*}

\paragraph{Trained comparison routes.}
The trained lane uses the official SD-Inpaint snapshot
\citep{stableDiffusionInpaintingModelCard} with 50 DDIM steps, CFG 7.5, and
FP16.  BrushNet \citep{ju2024brushnet} uses official repository commit
\texttt{0f9d9e54}, its released random-mask checkpoint, the frozen SD1.5 base,
50 UniPC steps, CFG 7.5, conditioning scale 1.0, and FP16.  PixelHacker
\citep{xu2025pixelhacker} uses official repository commit \texttt{f5567db2}, the
\texttt{ft\_places2} checkpoint and supplied VAE, 20 DDIM steps, strength
0.999, noise offset 0.0357, CFG 4.5, FP32, and no text prompt.
PixelHacker retains this fixed configuration on AFHQ and CelebA-HQ.  Full
environment pins will accompany the released code.

LanPaint and PILOT process the frozen $512\!\times\!512$, batch-one Main35 cases
with their fixed text conditions and per-case seeds.  Unknown model-input
pixels are filled with uint8 value 127; masks are nearest-neighbor resized and
rebinarized, with the required inversion for PILOT's native known-background
convention.  Unknown-region ground truth does not enter inference.  After
inference, observed known-region pixels are copied exactly into the evaluated
composite.  All 3,500 LanPaint and PILOT cases completed without inference
failure.  Exact known-region equality is therefore a shared evaluation-interface
property rather than an algorithm-specific result.

Because Step-PI, LanPaint, and PILOT retain different native inference routes
and were not assigned a common tuning budget, their results provide descriptive
complete-system positioning rather than tuning-matched or algorithm-intrinsic
superiority.  The 35-case resource audit remeasures the same frozen routes on
35 shared inputs and an A100 GPU only as descriptive full-route cost.  Complete
per-dataset quality
positioning and its exceptions are reported in Appendix~E.

\FloatBarrier

\raggedbottom

\section{Matched Mechanism, Robustness, and Text-Condition Evidence}
\label{sec:supp-extended-results}

This appendix extends the matched evidence in the main paper with complete
five-metric Main35 results, paired five-seed mechanism studies, one-field
release-schedule interventions, and a matched prompt-presence test.  All
analyses use frozen inputs and the common evaluation protocol specified in
Appendix~B.  Appendix~E records the corresponding computational cost and
reproducibility scope.

\subsection{Matched Main35 State and Joint-Release Effects}
\label{sec:supp-main35-matched-results}

Metric definitions, direction normalization, native pairing, and the
case-within-protocol stratified-bootstrap contract appear in
Sec.~\ref{sec:supp-main35-relative-gains}.  In the field-identical
P-Guidance-to-PI comparison, both methods share boundary and interior
objectives, current-step gradients, gains, normalization, safeguards, action
construction and placement, and uniform release; only PI carries discounted
state across reverse steps.  PI and Step-PI then share that complete persistent
controller and differ only in the four fields of the predefined release
schedule.  The two comparisons therefore isolate persistent state and the
complete predefined release schedule, respectively, over all 3,500 native
Main35 pairs.

\begin{table}[H]
\centering
\scriptsize
\setlength{\tabcolsep}{3.5pt}
\textbf{(a) Equal-dataset effects and pooled counts.}\par\smallskip
\begin{tabular}{lccr}
\toprule
Metric & Mean ($\times 100$) & 95\% CI ($\times 100$) & W/T/L \\
\midrule
Masked L1 & +0.385 & [+0.336, +0.434] & 2143/0/1357 \\
Boundary L1 & +0.144 & [+0.131, +0.158] & 2363/0/1137 \\
Masked LPIPS & +0.107 & [+0.089, +0.124] & 2135/0/1365 \\
Boundary LPIPS & +0.311 & [+0.278, +0.345] & 2332/0/1168 \\
CLIP-Q & +0.523 & [+0.196, +0.850] & 1996/0/1504 \\
\bottomrule
\end{tabular}
\par\medskip
\textbf{(b) Within-dataset PI win rates (\%).}\par\smallskip
\begin{tabular}{lrrr}
\toprule
Metric & AFHQ & CelebA-HQ & Places2 \\
\midrule
Masked L1 & 59.23 & 70.14 & 48.88 \\
Boundary L1 & 55.77 & 85.71 & 54.75 \\
Masked LPIPS & 56.15 & 70.29 & 52.62 \\
Boundary LPIPS & 59.77 & 78.64 & 56.75 \\
CLIP-Q & 56.54 & 57.71 & 56.62 \\
\bottomrule
\end{tabular}
\caption{Direction-normalized paired effects of enabling cross-step persistent state (PI versus P-Guidance) on Main35. Means and 95\% case-within-protocol stratified-bootstrap intervals resample cases within protocol, weight protocols equally within dataset, and weight datasets equally. Values are multiplied by 100 for readability; positive values favor PI. Wins/ties/losses are descriptive pooled counts over the same 3,500 pairs; panel (b) reports the complete descriptive dataset win rates.}
\label{tab:supp-main35-persistent-state-paired}
\end{table}

All five direction-normalized mean effects favor PI, and all five 95\%
stratified-bootstrap intervals exclude zero.  The pooled per-case win rates
range from 57.03\% to 67.51\%, showing broad but not case-universal
improvement.  Together, the mean effects, intervals, and paired win rates
support a positive aggregate persistent-state effect across Main35.

\begin{table}[t]
\centering
\small
\begin{tabular}{lcc}
\toprule
Metric & Mean ($\times 100$) & 95\% CI ($\times 100$) \\
\midrule
Masked L1 & +1.253 & [+1.165, +1.339] \\
Boundary L1 & +0.940 & [+0.918, +0.964] \\
Masked LPIPS & +0.550 & [+0.519, +0.581] \\
Boundary LPIPS & +1.483 & [+1.423, +1.543] \\
CLIP-Q & +2.436 & [+1.982, +2.893] \\
\bottomrule
\end{tabular}
\caption{Direction-normalized equal-dataset paired effects of Step-PI versus PI on Main35. The 95\% case-within-protocol stratified-bootstrap intervals resample cases within fixed protocols. Values and interval endpoints are multiplied by 100 for readability; positive values favor Step-PI.}
\label{tab:main35-paired-step-pi-vs-pi}
\end{table}

\begin{table}[t]
\centering
\small
\begin{tabular}{llrr}
\toprule
Dataset & Metric & Pairs & Win rate \\
\midrule
AFHQ & Masked L1 & 1,300 & 66.23\% \\
 & Boundary L1 & 1,300 & 99.23\% \\
 & Masked LPIPS & 1,300 & 79.08\% \\
 & Boundary LPIPS & 1,300 & 89.08\% \\
 & CLIP-Q & 1,300 & 51.00\% \\
\midrule
CelebA-HQ & Masked L1 & 1,400 & 77.07\% \\
 & Boundary L1 & 1,400 & 96.71\% \\
 & Masked LPIPS & 1,400 & 73.14\% \\
 & Boundary LPIPS & 1,400 & 83.07\% \\
 & CLIP-Q & 1,400 & 61.64\% \\
\midrule
Places2 & Masked L1 & 800 & 77.88\% \\
 & Boundary L1 & 800 & 98.88\% \\
 & Masked LPIPS & 800 & 81.75\% \\
 & Boundary LPIPS & 800 & 90.88\% \\
 & CLIP-Q & 800 & 52.12\% \\
\bottomrule
\end{tabular}
\caption{Within-dataset paired win rates of Step-PI over PI for all five metrics in the native paired Main35 analysis.}
\label{tab:supp-main35-all-win-rates}
\end{table}

\FloatBarrier

\subsection{Five-Seed Matched Mechanism Studies on MS175}
\label{sec:supp-multiseed-protocol}
\label{sec:supp-multiseed-robustness}

\paragraph{Shared cohort and statistics.}
The 175-case protocol-stratified random subset (MS175) is an outcome-blind
subset of Main35.
Before any method output or evaluation metric was examined, five unique-source
cases were randomly sampled within each of the 35 protocols, yielding 175
conditions (65 AFHQ, 70 CelebA-HQ, and 40 Places2).  The resulting cohort was
then fixed and reused across both studies.  Each case uses its fixed seed plus
four deterministically derived seeds.  Both studies pair image, mask,
prompt, seed, initial latent, known-region noise, sampler, objectives, gains,
safeguards, action construction, and placement within every comparison; each
contains 875 method pairs.

For each of the five metrics, the hierarchy averages seeds within case, cases
within protocol, protocol means within dataset, and then datasets equally.  It
uses 10,000 case-cluster bootstrap draws, a nested seed-resampling sensitivity
analysis, 100,000 case-cluster sign flips, and Holm correction across the five
metrics.  These studies vary inference initialization rather than training seeds.

\paragraph{Persistent-state study.}
P-Guidance and PI are field-identical except that PI carries discounted state
across reverse steps.  The evaluator pairs the retained P-Guidance and
PI outputs under identical case and seed identities.  All 875 pairs and all
five metrics pass the completeness and pairing checks.

\begin{table*}[!t]
\centering
\scriptsize
\setlength{\tabcolsep}{0.45em}
\resizebox{\textwidth}{!}{%
\begin{tabular}{l r r c c c c}
\toprule
Metric & P-Guidance & PI & $\Delta$ [95\% CI] & Holm $p$ & Case W/T/L & Endpoint W/T/L \\
\midrule
Boundary L1 $\downarrow$ & 0.042021 & 0.038400 & $+0.003621$ [0.003551, 0.003691] & $5.0\!\times\!10^{-5}$ & 175/0/0 & 814/0/61 \\
Boundary LPIPS $\downarrow$ & 0.127008 & 0.117900 & $+0.009108$ [0.008866, 0.009353] & $5.0\!\times\!10^{-5}$ & 175/0/0 & 749/0/126 \\
Masked L1 $\downarrow$ & 0.176841 & 0.165000 & $+0.011841$ [0.011584, 0.012102] & $5.0\!\times\!10^{-5}$ & 175/0/0 & 778/0/97 \\
Masked LPIPS $\downarrow$ & 0.125383 & 0.118000 & $+0.007383$ [0.007199, 0.007572] & $5.0\!\times\!10^{-5}$ & 175/0/0 & 782/0/93 \\
CLIP-Q $\uparrow$ & 0.531467 & 0.561900 & $+0.030433$ [0.029674, 0.031199] & $5.0\!\times\!10^{-5}$ & 175/0/0 & 626/0/249 \\
\bottomrule
\end{tabular}%
}
\caption{Persistent-state robustness across five stochastic inference initializations on MS175.  Method columns follow the frozen seed--case--protocol--dataset hierarchy; positive direction-normalized $\Delta$ favors PI.  All five intervals exclude zero, all five Holm-adjusted tests remain significant, and every metric has a positive effect in each seed-wise estimate.  Case W/T/L compares five-seed case means; endpoint W/T/L compares all 875 case--seed pairs.}
\label{tab:supp-multiseed-persistent-state-pooled}
\end{table*}

All five direction-normalized hierarchical effects favor PI, and all five
95\% intervals exclude zero after the paired five-seed evaluation.  Every
metric is positive in each of the five seed-wise estimates; the table also
reports win/tie/loss counts for both the five-seed case means and all 875 endpoints.  This
supports a consistent positive incremental effect of persistent state across
the five stochastic initializations within the frozen MS175 and SD1.5 scope,
rather than a claim about training seeds or untested backbones and samplers.

\paragraph{Joint-release study.}
The independent release study pairs PI and Step-PI while changing only the
complete predefined release schedule.  Its 875 pairs comprise 1,750 completed
runs with no failures, retries, or outcome-based selection.  It is a separate
five-seed study on the fixed MS175 cohort rather than a repeated-seed expansion
of Main35.  Positive effects below favor Step-PI.

\begin{table*}[!t]
\centering
\scriptsize
\setlength{\tabcolsep}{0.55em}
\begin{tabular}{l r r c r c}
\toprule
Metric & PI & Step-PI & $\Delta$ [95\% CI] & Relative & Positive seeds \\
\midrule
Boundary L1 $\downarrow$ & 0.038400 & 0.029267 & $+0.009133$ [0.008583, 0.009702] & $+23.78\%$ & 5/5 \\
Boundary LPIPS $\downarrow$ & 0.117900 & 0.103067 & $+0.014833$ [0.013234, 0.016383] & $+12.58\%$ & 5/5 \\
Masked L1 $\downarrow$ & 0.165000 & 0.152433 & $+0.012567$ [0.010238, 0.014918] & $+7.62\%$ & 5/5 \\
Masked LPIPS $\downarrow$ & 0.118000 & 0.112300 & $+0.005700$ [0.004925, 0.006499] & $+4.83\%$ & 5/5 \\
CLIP-Q $\uparrow$ & 0.561900 & 0.586767 & $+0.024867$ [0.015148, 0.034914] & $+4.43\%$ & 5/5 \\
\bottomrule
\end{tabular}
\caption{MS175 robustness across stochastic inference initialization.  Positive $\Delta$ favors Step-PI after metric-direction normalization.  Each metric has five positive seed sweeps, and every pooled interval excludes zero.  All five two-sided case-cluster randomization tests attain $p=9.9999\!\times\!10^{-6}$ at the finite 100,000-draw resolution and $p_{\mathrm{Holm}}=4.99995\!\times\!10^{-5}$.}
\label{tab:supp-multiseed-pooled}
\end{table*}

\begin{table*}[!t]
\centering
\scriptsize
\setlength{\tabcolsep}{0.7em}
\begin{tabular}{l l r r c c}
\toprule
Dataset & Metric & PI & Step-PI & $\Delta$ [95\% CI] & Positive seeds \\
\midrule
AFHQ & Boundary L1 $\downarrow$ & 0.0407 & 0.0322 & $+0.0085$ [0.0075, 0.0096] & 5/5 \\
AFHQ & Boundary LPIPS $\downarrow$ & 0.1098 & 0.0982 & $+0.0116$ [0.0088, 0.0143] & 5/5 \\
AFHQ & Masked L1 $\downarrow$ & 0.1763 & 0.1689 & $+0.0074$ [0.0028, 0.0119] & 5/5 \\
AFHQ & Masked LPIPS $\downarrow$ & 0.1444 & 0.1376 & $+0.0068$ [0.0044, 0.0091] & 5/5 \\
AFHQ & CLIP-Q $\uparrow$ & 0.6988 & 0.7115 & $+0.0127$ [0.0003, 0.0248] & 5/5 \\
\midrule
CelebA-HQ & Boundary L1 $\downarrow$ & 0.0315 & 0.0237 & $+0.0078$ [0.0069, 0.0087] & 5/5 \\
CelebA-HQ & Boundary LPIPS $\downarrow$ & 0.1087 & 0.0924 & $+0.0163$ [0.0124, 0.0202] & 5/5 \\
CelebA-HQ & Masked L1 $\downarrow$ & 0.1450 & 0.1281 & $+0.0169$ [0.0114, 0.0226] & 5/5 \\
CelebA-HQ & Masked LPIPS $\downarrow$ & 0.0729 & 0.0692 & $+0.0037$ [0.0025, 0.0049] & 5/5 \\
CelebA-HQ & CLIP-Q $\uparrow$ & 0.5221 & 0.5716 & $+0.0495$ [0.0311, 0.0683] & 5/5 \\
\midrule
Places2 & Boundary L1 $\downarrow$ & 0.0430 & 0.0319 & $+0.0111$ [0.0097, 0.0126] & 5/5 \\
Places2 & Boundary LPIPS $\downarrow$ & 0.1352 & 0.1186 & $+0.0166$ [0.0118, 0.0208] & 5/5 \\
Places2 & Masked L1 $\downarrow$ & 0.1737 & 0.1603 & $+0.0134$ [0.0085, 0.0182] & 5/5 \\
Places2 & Masked LPIPS $\downarrow$ & 0.1367 & 0.1301 & $+0.0066$ [0.0050, 0.0083] & 5/5 \\
Places2 & CLIP-Q $\uparrow$ & 0.4648 & 0.4772 & $+0.0124$ [-0.0109, 0.0381] & 5/5 \\
\bottomrule
\end{tabular}
\caption{Per-dataset MS175 effects.  Positive $\Delta$ favors Step-PI.  All 15 point estimates are positive and all are positive in 5/5 seed sweeps.  Fourteen intervals exclude zero; the sole exception is Places2 CLIP-Q.}
\label{tab:supp-multiseed-dataset}
\end{table*}

All 25 seed--metric equal-dataset point estimates favor Step-PI.  All 15
dataset--metric point estimates also favor Step-PI, and 14 of their intervals
exclude zero.  The sole exception is Places2 CLIP-Q: its positive point
estimate has an interval that crosses zero.  The result therefore supports
joint-release robustness across five stochastic initializations on the frozen
MS175 cohort, not a Main35-wide or universal per-case guarantee.
\FloatBarrier

\subsection{Five-Metric Conditional Sensitivity of the Predefined Release Schedule}
\label{sec:supp-release-loo-protocol}
\label{sec:supp-release-loo-results}

\paragraph{Protocol.}
The seed-0 study reuses the same fixed MS175 cohort and contains uniform PI,
full Step-PI, and four one-field-uniformization arms for $r_B$, $q$, $h$, and
$r_I$.  Each intervention sets only the selected field to one while the other
three retain their predefined Step-PI trajectories.  It never fixes a field
to zero: both objectives, gradients, persistent states, actions, and placements
remain active.  All 1,050 case--arm runs completed without failure, and all
non-release fields are paired.

Positive effects favor full Step-PI over the corresponding uniformized arm.
Intervals use 10,000 paired stratified-bootstrap draws over the fixed seed-0
case outputs.  The complete study covers the same five metrics used in the
main paper.

\begin{table*}[!t]
\centering
\scriptsize
\setlength{\tabcolsep}{0.9em}
\begin{tabular}{l l c c}
\toprule
Uniformized field & Metric & $\Delta$ [95\% CI] & W/T/L \\
\midrule
$r_B\!\rightarrow\!1$ & Boundary L1 $\downarrow$ & $+0.000823$ [0.000786, 0.000860] & 175/0/0 \\
 & Boundary LPIPS $\downarrow$ & $+0.002876$ [0.002748, 0.003006] & 175/0/0 \\
 & Masked L1 $\downarrow$ & $+0.002350$ [0.002263, 0.002439] & 175/0/0 \\
 & Masked LPIPS $\downarrow$ & $+0.001560$ [0.001507, 0.001616] & 175/0/0 \\
 & CLIP-Q $\uparrow$ & $+0.002273$ [0.002142, 0.002403] & 175/0/0 \\
\midrule
$q\!\rightarrow\!1$ & Boundary L1 $\downarrow$ & $+0.000470$ [0.000451, 0.000490] & 175/0/0 \\
 & Boundary LPIPS $\downarrow$ & $+0.003785$ [0.003624, 0.003948] & 175/0/0 \\
 & Masked L1 $\downarrow$ & $+0.004029$ [0.003868, 0.004191] & 175/0/0 \\
 & Masked LPIPS $\downarrow$ & $+0.001135$ [0.001102, 0.001169] & 175/0/0 \\
 & CLIP-Q $\uparrow$ & $+0.001989$ [0.001881, 0.002097] & 175/0/0 \\
\midrule
$h\!\rightarrow\!1$ & Boundary L1 $\downarrow$ & $+0.000196$ [0.000188, 0.000205] & 175/0/0 \\
 & Boundary LPIPS $\downarrow$ & $+0.001060$ [0.001017, 0.001103] & 175/0/0 \\
 & Masked L1 $\downarrow$ & $+0.001679$ [0.001613, 0.001745] & 175/0/0 \\
 & Masked LPIPS $\downarrow$ & $+0.000426$ [0.000412, 0.000439] & 175/0/0 \\
 & CLIP-Q $\uparrow$ & $+0.003978$ [0.003756, 0.004201] & 175/0/0 \\
\midrule
$r_I\!\rightarrow\!1$ & Boundary L1 $\downarrow$ & $+0.002351$ [0.002255, 0.002451] & 175/0/0 \\
 & Boundary LPIPS $\downarrow$ & $+0.013625$ [0.013003, 0.014261] & 175/0/0 \\
 & Masked L1 $\downarrow$ & $+0.015108$ [0.014492, 0.015723] & 175/0/0 \\
 & Masked LPIPS $\downarrow$ & $+0.004965$ [0.004798, 0.005139] & 175/0/0 \\
 & CLIP-Q $\uparrow$ & $+0.015344$ [0.014462, 0.016225] & 175/0/0 \\
\bottomrule
\end{tabular}
\caption{Five-metric conditional sensitivity to one-field uniformization on frozen seed-0 MS175.  Positive values favor full Step-PI over the corresponding uniformized arm while the other three fields retain their predefined trajectories.  All 20 intervals exclude zero, and the 175/0/0 entries are paired case wins/ties/losses.  These conditional contrasts are not an additive decomposition of four independent effective factors.}
\label{tab:supp-release-loo-all-metrics}
\end{table*}

Across the four one-field interventions and all five metrics, every one of the
20 mean contrasts favors full Step-PI, every case-resampling interval excludes
zero, and every contrast has 175/0/0 paired case wins/ties/losses.  Thus,
replacing any one predefined field by uniform release worsens all five
evaluated metrics while the other three fields retain their predefined
trajectories.

These are conditional interventions within the implemented four-field
schedule.  In particular, $q_k$ and $r_I(t_k)$ are serial interior scalings,
so the four interventions are not an independent additive decomposition and
do not compare the implemented schedule against alternative shared nonuniform
schedules.  We therefore interpret the results as five-metric conditional
sensitivity, not as proof of schedule minimality or global optimality.
\FloatBarrier

\subsection{Prompt-Presence Responsiveness Under a Text-Only Intervention}
\label{sec:supp-text-condition-protocol}
\label{sec:supp-text-condition-results}

\begin{table}[!b]
\centering
\scriptsize
\setlength{\tabcolsep}{0.35em}
\begin{tabularx}{\columnwidth}{@{}l r X r@{}}
\toprule
Scope & $n$ & Prompt-CLIP effect [95\% CI] & W/T/L \\
\midrule
Equal-dataset macro & 175 & $+0.01366$ [0.01081, 0.01658] & 125/0/50 \\
AFHQ & 65 & $+0.00403$ [0.00132, 0.00677] & 36/0/29 \\
CelebA-HQ & 70 & $+0.00836$ [0.00549, 0.01140] & 55/0/15 \\
Places2 & 40 & $+0.02858$ [0.02096, 0.03622] & 34/0/6 \\
\bottomrule
\end{tabularx}
\caption{Paired Prompt-CLIP Correct-minus-Empty effects on MS175 under otherwise fixed inference conditions.  Positive values favor Correct; intervals use 10,000 paired stratified-bootstrap draws and W/T/L counts paired case effects.  The aggregate weights protocols and datasets equally; per-dataset rows weight protocols equally.  Prompt-CLIP is distinct from the no-reference Main35 CLIP-Q metric.}
\label{tab:supp-text-condition-primary}
\end{table}

\paragraph{Protocol.}
The seed-0 study reuses the same protocol-stratified random MS175 cohort and
holds the complete evaluated Step-PI pipeline fixed while changing only
positive-text presence.  The paired 175-case analysis uses 350
endpoints---175 \emph{Correct} and 175 \emph{Empty}---with one matched pair per
case.  Image, mask, seed, initial latent, known-region noise, sampler,
controller, and every non-text field are paired.  Prompt mappings and the case
set were fixed before inference.  Appendix~D presents a separate six-case
strict-random qualitative display; the complete paired study here provides the
quantitative evidence.

Prompt-CLIP uses frozen \texttt{openai/clip-vit-base-patch32}
\citep{radford2021clip} on a fixed padded unknown-region crop.  For the Correct
and Empty outputs $x_i^C$ and $x_i^E$, respectively, and correct text $t_i^C$,
the presence effect is
\begin{equation}
\delta_i^{\mathrm{pres}}
  =s(x_i^C,t_i^C)-s(x_i^E,t_i^C).
\end{equation}
The aggregate weights protocols and datasets equally, and intervals use 10,000
paired stratified-bootstrap draws.  Prompt-CLIP is the paired text--output
alignment endpoint for this intervention and is distinct from the no-reference
CLIP-Q metric used in the Main35 benchmark.  Static inspection found no
unknown-region ground truth in inference; the reproduction protocol and
reproducibility scope are reported in Appendices~B and~E.

\paragraph{Results.}
The Correct condition yields significantly higher paired Prompt-CLIP alignment than the Empty condition in
the equal-dataset aggregate.  The AFHQ, CelebA-HQ, and Places2 effects are also
positive, and all four 95\% intervals exclude zero.  Under otherwise fixed
inference conditions, these matched effects show that the user-provided text
condition is functionally active rather than ignored by the frozen SD1.5 and
Step-PI pipeline.

Prompt-CLIP measures supplied-text alignment, whereas reconstruction metrics
compare against one held-out realization.  Because inpainting admits multiple
plausible completions, stronger text alignment need not reduce that distance;
we therefore claim text-condition responsiveness, not uniformly better
reference reconstruction.
\FloatBarrier

\onecolumn
\raggedbottom

\section{Qualitative Evidence and Stress-Test Diagnostics}
\label{sec:supp-qualitative-atlas}

This section provides visual counterparts to the quantitative evidence in the
main paper and Appendix~C: random inspection of the internal construction
ladder, strict-random prompt-presence examples, random native-route system
comparisons, and cross-domain stress-test diagnostics.  The matched
quantitative analyses carry the corresponding statistical and mechanism
evidence.

Every grid follows one reading convention: each row retains the same source
image and mask across methods, compared methods are adjacent, and exact case
identifiers remain visible.  Captions identify the displayed cohort, fixed
conditions, intended reading, and claim boundary.  Throughout this
appendix, opaque violet denotes the unknown region and the thin mint contour is
display-only; neither exposes hidden pixels nor changes the model input.

\subsection{Random Construction-Ladder Inspection}

The shared-foundation grid visualizes the internal construction ladder on
three cases randomly drawn from the evaluated cases, one per dataset, without
conditioning on any method output.  P-Guidance$\rightarrow$PI visually
accompanies the persistent-state comparison, whereas
PI$\rightarrow$Step-PI accompanies the predefined-release comparison.  Exact
case identifiers are shown for traceability.  These are the same three random
cases used in the main paper's native-route external
comparison.  The plate makes the successive interfaces inspectable; the
corresponding five-metric matched analyses, rather than a presumed visual
monotonic ordering, support the two conclusions.

\begin{figure}[H]
  \centering
  \includegraphics[width=\textwidth]{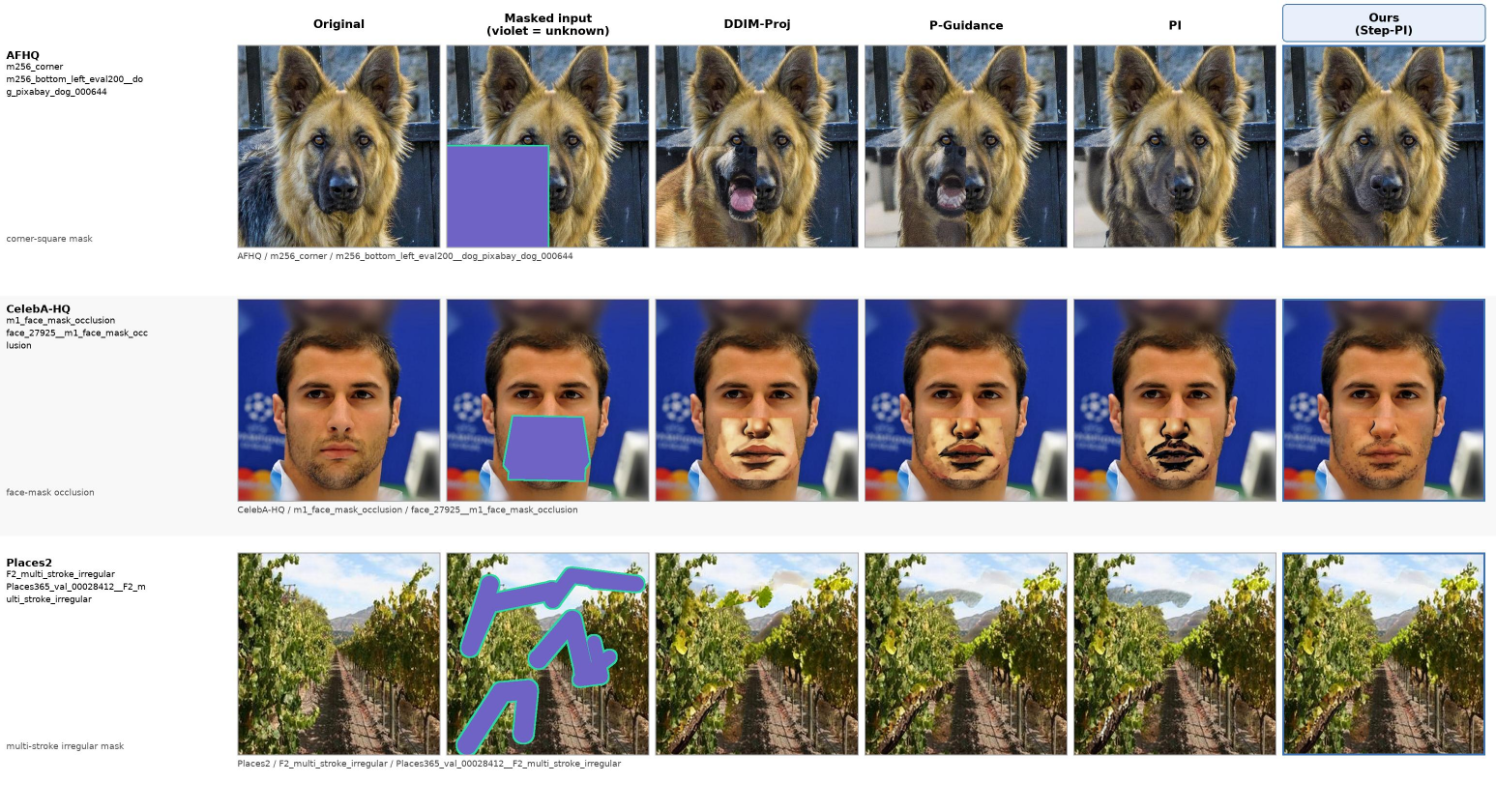}
  \caption{Random shared-foundation construction-ladder inspection on three
  cases.  The columns trace DDIM-Proj $\rightarrow$ P-Guidance $\rightarrow$
  PI $\rightarrow$ Step-PI under the same frozen backbone, and exact case IDs
  are printed per row.  These are the same three random cases shown in the
  main paper's native-route external comparison, without method-specific
  reselection.  The plate visually accompanies the five-metric matched
  comparisons.}
  \label{fig:supp-internal-inpainting-grid}
\end{figure}

\FloatBarrier

\subsection{Strict-Random Prompt-Presence Examples}
\label{sec:supp-outcome-blind-core}
\label{sec:supp-text-condition-qualitative}

The display contains six MS175 cases drawn strictly at random before Correct or
Empty outcomes were inspected, two per dataset.  Within each paired row, every
inference input and Step-PI setting is fixed; only the positive text condition
changes.  The Original column provides evaluation-only context unavailable during
inference, while the Masked column follows the shared display convention without altering model inputs.
The paired Prompt-CLIP analysis in Appendix~C provides the three-domain
quantitative evidence; this plate is its strict-random visual companion and
does not test fine-grained counterfactual control or human preference.

\begin{figure}[H]
  \centering
  \includegraphics[width=0.76\textwidth]{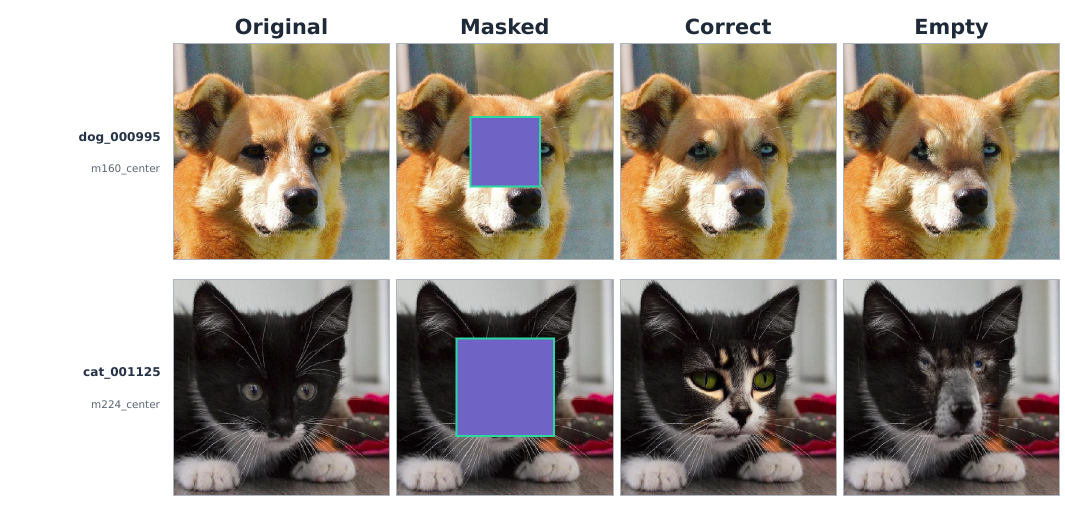}
  \par\smallskip
  \includegraphics[width=0.76\textwidth]{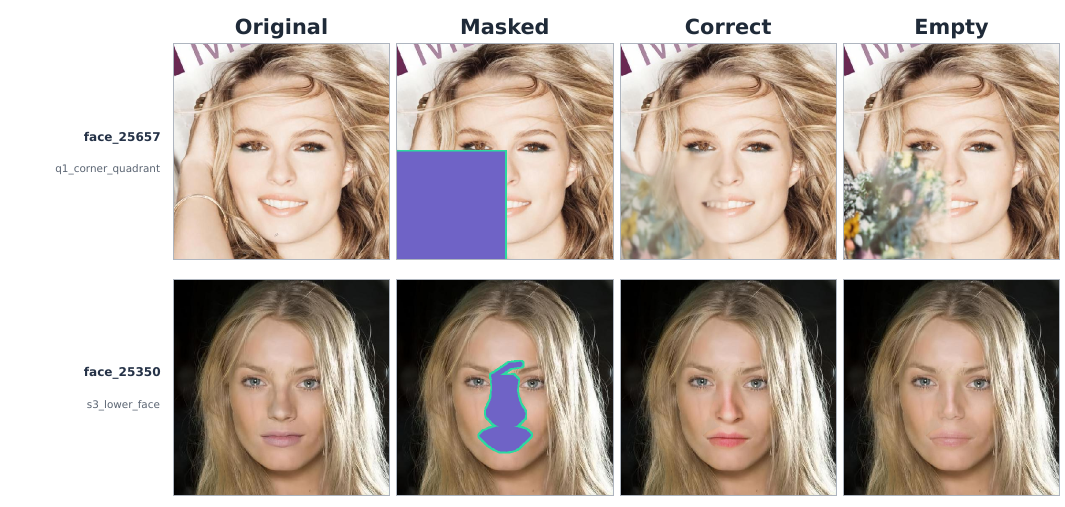}
  \par\smallskip
  \includegraphics[width=0.76\textwidth]{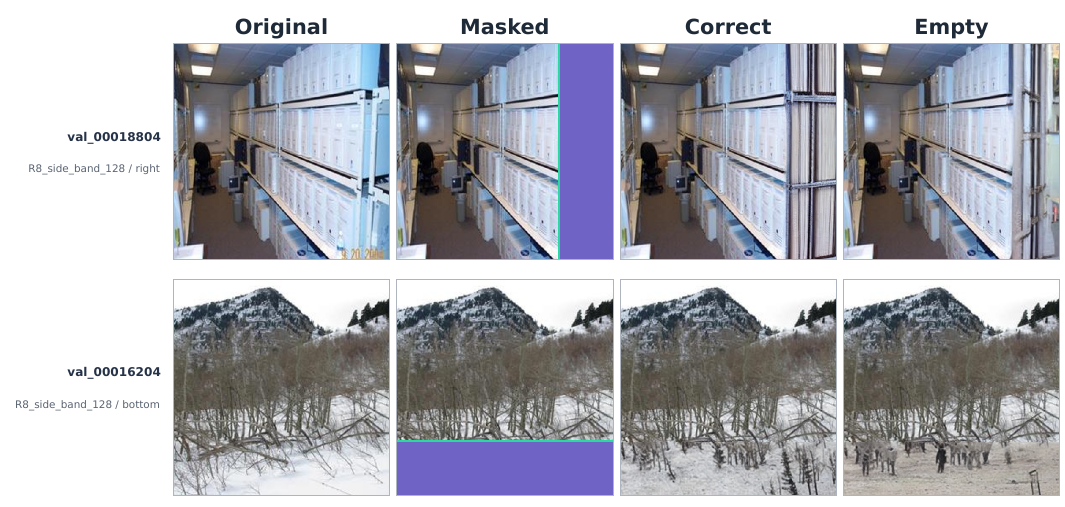}
  \caption{Prompt-presence responses on six strict-random cases: AFHQ
  (top), CelebA-HQ (middle), and Places2 (bottom), with two cases per dataset.
  Every panel uses the shared four-column order Original, Masked, Correct, and
  Empty.  Original is evaluation-only reader context and was unavailable
  during inference; Masked uses the shared violet/mint display convention; and
  only the positive text condition differs between Correct and Empty.  The
  corresponding aggregate and per-dataset estimates are reported in
  Appendix~C; this plate is their strict-random visual companion.}
  \label{fig:supp-text-condition-presence}
\end{figure}

\FloatBarrier

\subsection{Random Native-Route System Comparisons}

Each plate contains six cases randomly drawn per dataset without conditioning
on any method output, for 18 cases in total; one per dataset is reused in the
main paper's first-page teaser.  Every row fixes the source image and mask,
while methods retain their native inference routes.  The held-out Original is
shown as reference context and is not assumed to be the unique valid
completion.  These plates provide outcome-independent inspection of complete
systems.  Because their native interfaces and inference procedures differ,
they are descriptive system-level comparisons rather than intervention-matched
attribution, formal population-frequency estimates, or human-preference
evidence.

\begin{figure}[H]
  \centering
  \includegraphics[width=0.95\textwidth]{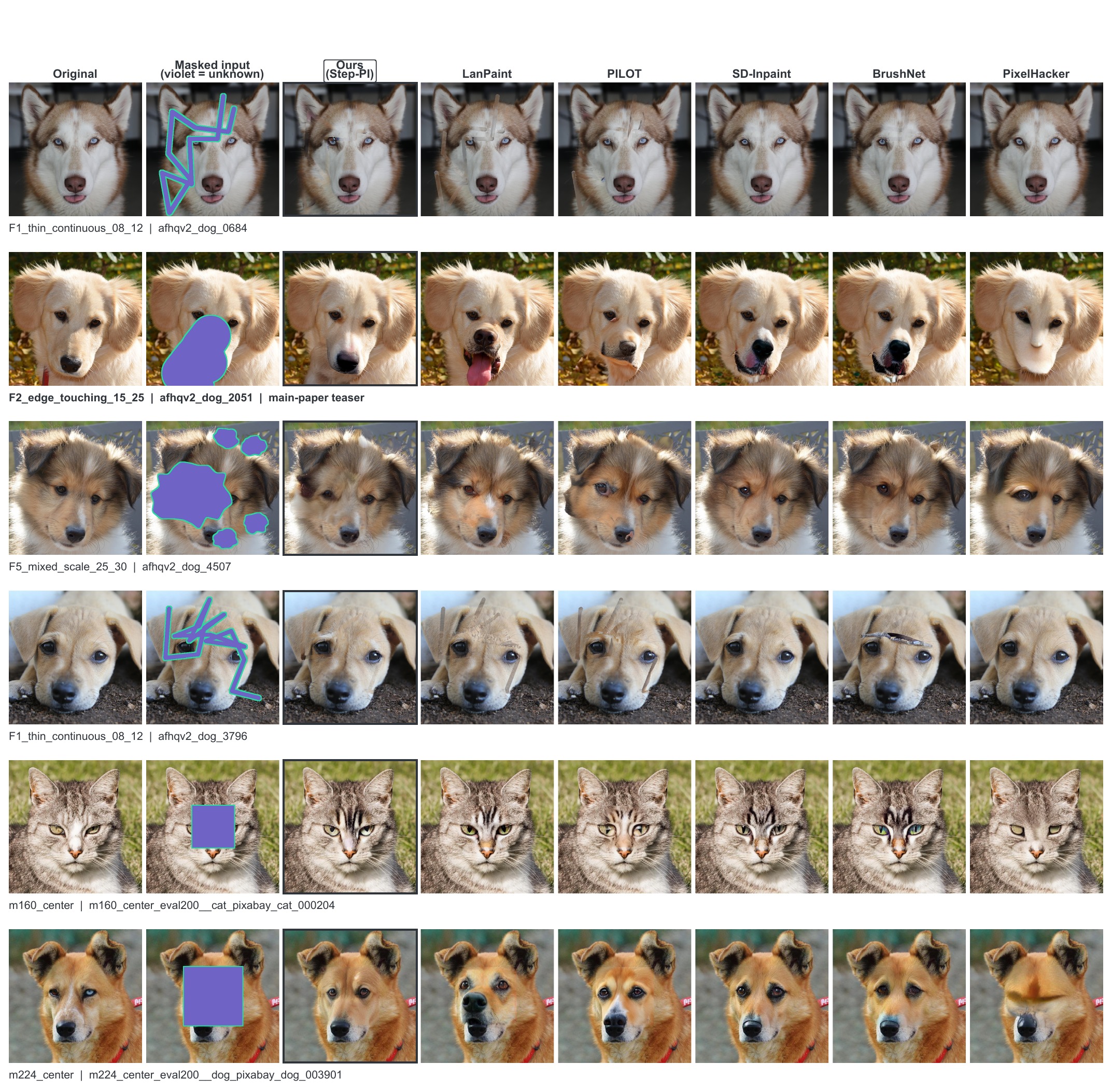}
  \caption{Native-route comparison on six randomly drawn AFHQ cases.  Each row
  shows one source image and mask across Step-PI, LanPaint, PILOT, SD-Inpaint,
  BrushNet, and PixelHacker; the case label marks the row reused in the main
  paper's first-page teaser.  Exact protocol and case IDs are printed below
  each row.  The draw is independent of all method outputs, and every method
  retains its native inference route.}
  \label{fig:supp-external-atlas-afhq}
\end{figure}

\begin{figure}[H]
  \centering
  \includegraphics[width=\textwidth]{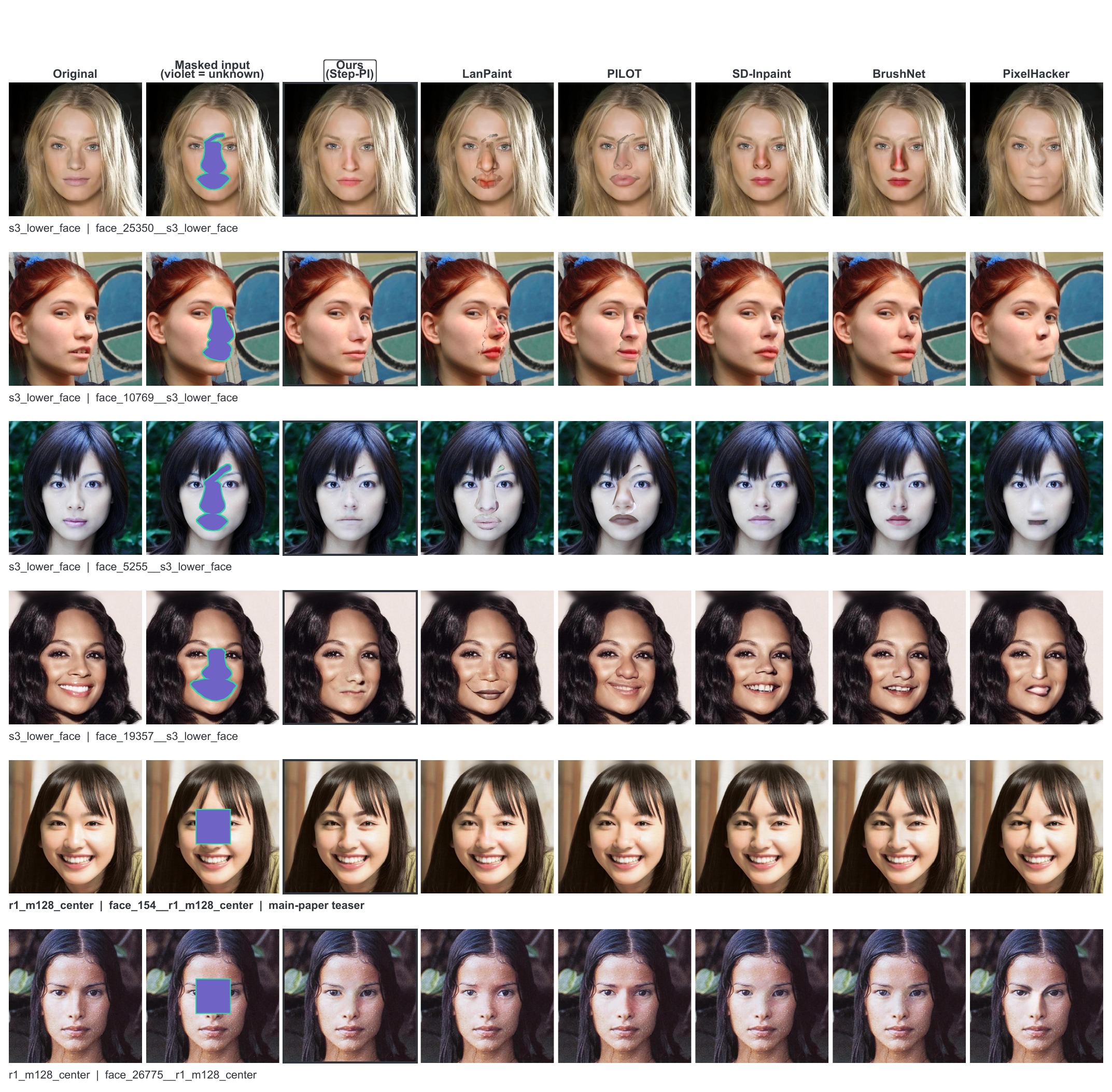}
  \caption{Native-route comparison on six randomly drawn CelebA-HQ cases,
  following the same column order and reading convention as
  Fig.~\ref{fig:supp-external-atlas-afhq}.  All panels are independently
  verified full-resolution composites on the same row-specific image and mask,
  with exact identifiers printed per row.}
  \label{fig:supp-external-atlas-celebahq}
\end{figure}

\begin{figure}[p]
  \centering
  \includegraphics[width=\textwidth]{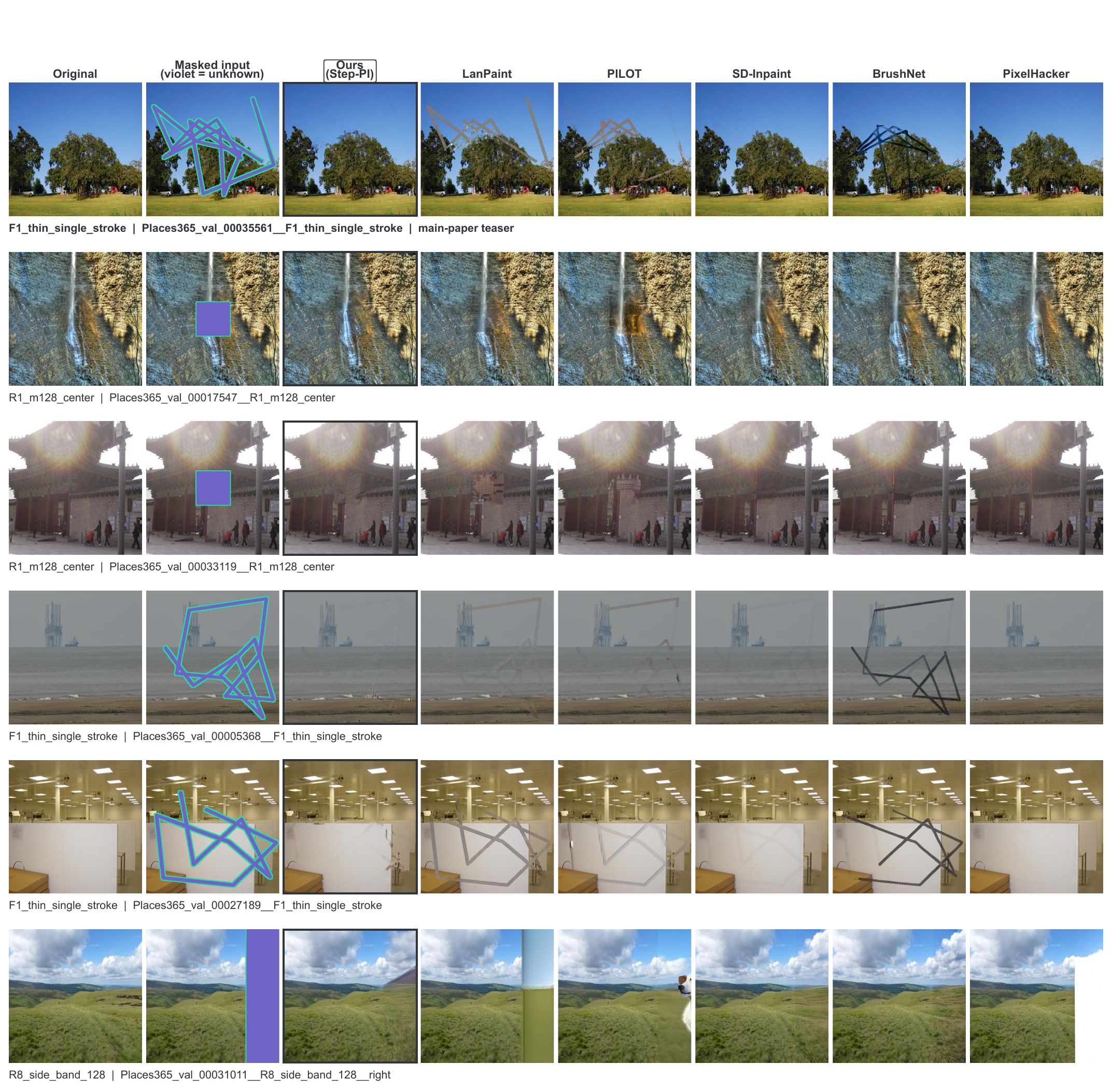}
  \caption{Native-route comparison on six randomly drawn Places2 cases.  Each
  method receives the same image and mask within a row while retaining its
  native inference route, and exact identifiers are printed per row.}
  \label{fig:supp-external-atlas-places2}
\end{figure}

\FloatBarrier

\subsection{Cross-Domain Stress-Test Diagnostics}

The final plate presents a separate fixed six-case cross-domain diagnostic set
spanning challenging mask--content configurations, with two cases per dataset.
For every case, the complete active comparison lane is retained without
method-specific inclusion or exclusion.  The plate is organized to make
diverse structural, semantic, and blending difficulties inspectable rather
than to estimate their prevalence.

Inspection focuses on structural continuation under disconnected or thin
masks, facial and semantic coherence under central occlusion, and seam and
geometry consistency over long or repeated structures.  These descriptions
direct visual inspection of the fixed cases; they are observed case attributes,
not experimentally isolated causes of any output.

\begin{figure}[H]
  \centering
  \includegraphics[width=0.86\textwidth]{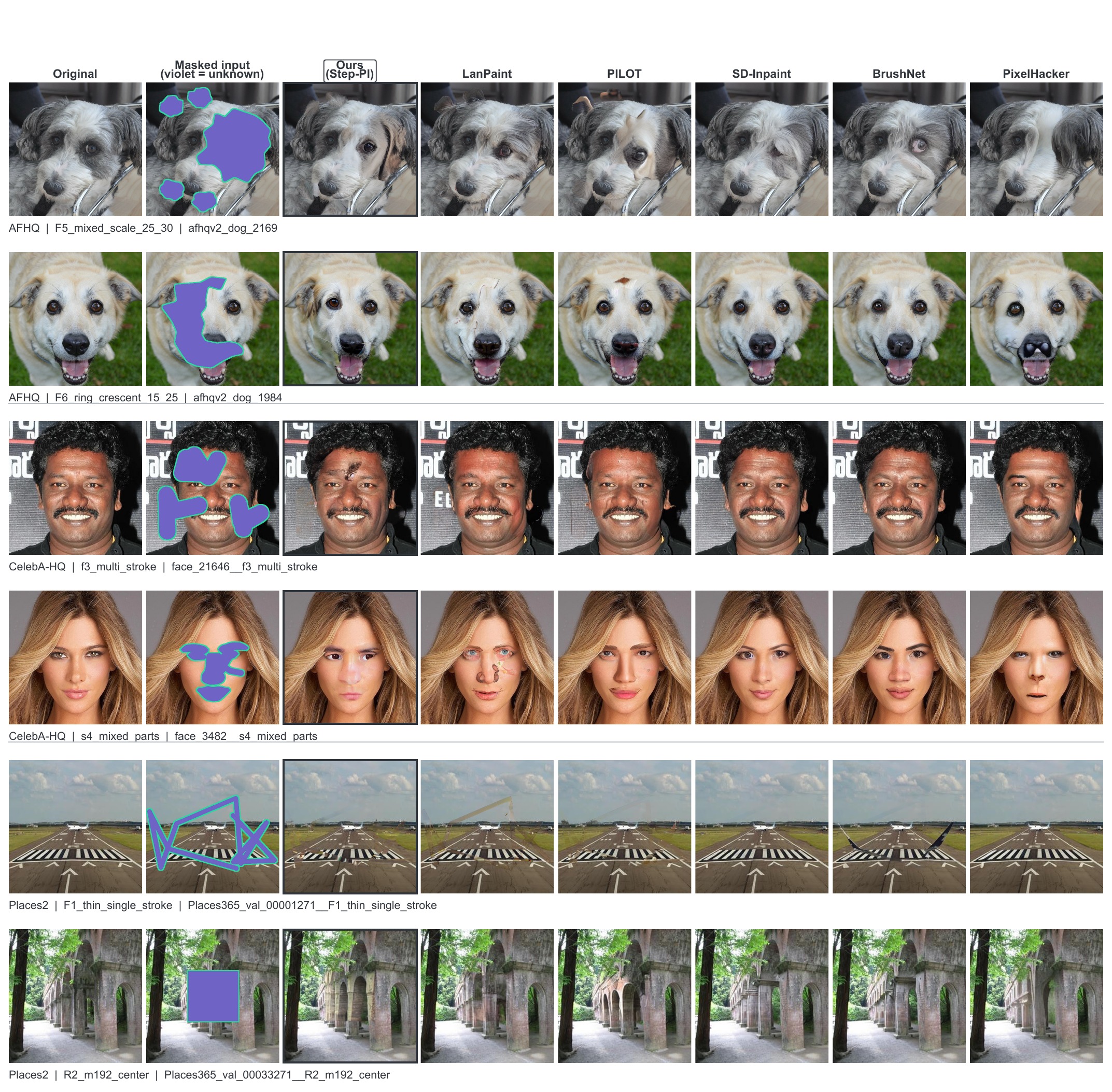}
  \caption{Cross-domain stress-test diagnostics, with two cases each from AFHQ,
  CelebA-HQ, and Places2 and the complete active comparison lane retained for
  every case.  Horizontal rules separate datasets, and exact protocol and case
  IDs are printed below every row.  The plate supports inspection of
  challenging structural, semantic, and blending behaviors rather than
  prevalence estimation or causal attribution.}
  \label{fig:supp-external-failure-atlas}
\end{figure}

\FloatBarrier

\section{Computational Cost and Evaluated Scope}
\label{sec:supp-efficiency-limitations}

\setcounter{dbltopnumber}{3}
\renewcommand{\dbltopfraction}{0.95}

\subsection{35-Case Runtime, Memory, and Structural Compute Audit}
\label{sec:supp-eff35-audit}

\paragraph{Cohort and measurement protocol.}
The resource audit fixes one outcome-independent case from each of the 35
Main35 protocols.
For nine methods, we measure batch-one, model-resident steady-state wall-clock
time on one NVIDIA A100-SXM4-80GB GPU.  Each method is executed in three
fresh-process blocks with three fixed warm-up cases per block; excluding the
warm-ups, this yields 945 formal timing tuples.  The analysis unit is the
method--case median over the three repeats.  Intervals use 10,000
dataset-stratified case/protocol bootstrap resamples (seed 20260820), with no
outlier deletion and no $p$- or $q$-values.  Peak device memory is collected at
20~Hz in a separate pass over the same 35 cases for all nine methods, yielding
315 method--case measurements; timing and memory are not joint observations.

\begin{table}[H]
  \centering
  \caption{Intervention-matched internal 35-case cost audit.  The upper panel
  reports empirical 2.5th and 97.5th percentiles of the 35 case medians;
  dataset-stratified bootstrap intervals govern the prose and
  paired comparisons.  Peak device memory is measured in a separate pass.
  The lower panel reports structural calls from a count-only trace, not from
  the timed runs.}
  \label{tab:supp-eff35-internal}
  \textbf{(a) Runtime and peak device memory.}\par\smallskip
  \resizebox{\textwidth}{!}{\begin{tabular}{lrrrrrr}
\toprule
Method & Steps & Feedback grad./step & Param. updates & Median s/img [empirical P2.5, P97.5] & Mean $\pm$ SD & Peak VRAM \\
\midrule
DDIM-Proj & 50 & 0 & 0 & 3.07 [3.06, 3.07] & 3.07 $\pm$ 0.01 & 4.18 GiB \\
P-Guidance & 50 & 2 & 0 & 30.48 [30.47, 30.49] & 30.47 $\pm$ 0.04 & 15.55 GiB \\
PI & 50 & 2 & 0 & 30.48 [30.47, 30.49] & 30.48 $\pm$ 0.06 & 15.55 GiB \\
Step-PI & 50 & 2 & 0 & 30.48 [30.47, 30.50] & 30.48 $\pm$ 0.04 & 15.55 GiB \\
\bottomrule
\end{tabular}
}
  \par\medskip
  \textbf{(b) Structural computation for the matched internal methods.}\par\smallskip
  \resizebox{0.88\textwidth}{!}{\begin{tabular}{lrrrrr}
\toprule
Method & Denoising steps & UNet calls & Autograd calls & Optimizer steps & Parameter updates \\
\midrule
DDIM-Proj & 50 & 50 & 0 & 0 & 0 \\
P-Guidance & 50 & 50 & 100 & 0 & 0 \\
PI & 50 & 50 & 100 & 0 & 0 \\
Step-PI & 50 & 50 & 100 & 0 & 0 \\
\bottomrule
\end{tabular}
}
\end{table}

\paragraph{Matched internal cost.}
The audit resolves no runtime difference between Step-PI and either matched
feedback-enabled route: the paired Step-PI-minus-P-Guidance difference is
$+0.008$~s/image $[-0.003,0.018]$, and the paired Step-PI-minus-PI difference is
$-0.002$~s/image $[-0.017,0.020]$.  P-Guidance, PI, and Step-PI share 50 UNet calls and 100 autograd
calls, with no optimizer step or parameter update.  The predefined release
schedule therefore adds no separately resolved cost beyond uniform PI in this
audit; this is not an equivalence claim.

The complete feedback route nevertheless remains computationally demanding.
Under the dataset-stratified bootstrap, Step-PI requires a median 30.48~s/image
$[30.47,30.50]$ and an independent peak-device-memory median of 15.55~GiB.
Relative to DDIM-Proj, its paired median overhead is $+27.41$~s/image
$[27.39,27.44]$, its runtime ratio is $9.94\times$ $[9.93,9.95]$, and its
peak device memory is $3.72\times$ as large.  The structural counts locate the
added computation in online differentiation: the feedback-enabled routes add
100 autograd calls to the same 50 UNet calls.  They do not attribute wall-clock
time to either objective or any individual suboperation.

\begin{table}[H]
  \centering
  \caption{Same-input, same-A100 descriptive full-route comparison on the
  35-case audit.
  All measured routes use frozen weights; online differentiation denotes
  latent or objective gradients rather than parameter updates.  Methods retain
  their benchmark-specific backbones, checkpoints, samplers, precision,
  conditioning interfaces, and inference routes.  The comparison therefore
  describes complete online cost on the same 35 cases and GPU rather than an
  intervention-matched efficiency comparison.  Peak device memory is measured
  in a separate pass.}
  \label{tab:supp-eff35-external}
  \resizebox{\textwidth}{!}{\begin{tabular}{llllrr}
\toprule
Method & Type & Native inference steps & Online differentiation & Median s/img [95\% CI] & Peak VRAM \\
\midrule
Step-PI & Training-free & 50 & 2 gradients/step & 30.48 [30.47, 30.50] & 15.55 GiB \\
LanPaint & Training-free & 30 Euler + 5 inner Langevin steps & None observed & 11.38 [11.34, 11.43] & 7.14 GiB \\
PILOT & Training-free & 100 & 100 latent gradients & 23.51 [23.45, 23.67] & 6.75 GiB \\
SD-Inpaint & Trained reference & 50 & None observed & 2.27 [2.26, 2.27] & 3.45 GiB \\
BrushNet & Trained reference & 50 & None observed & 2.94 [2.94, 2.95] & 4.97 GiB \\
PixelHacker & Trained reference & 20 & None observed & 1.33 [1.32, 1.33] & 7.46 GiB \\
\bottomrule
\end{tabular}
}
\end{table}

\paragraph{External descriptive positioning.}
The external rows are complete-route resource measurements, not causal
component estimates.  At the reported operating point, Step-PI is slower than
LanPaint and PILOT, while the trained reference routes are faster still; these
differences include each system's foundation, sampler, precision, conditioning
interface, and complete inference route.  They cannot be attributed to a
single algorithmic component.

Step-PI is therefore training-free but not computation-free.  This 35-case
audit characterizes one A100, $512\!\times\!512$, batch-one, 50-step DDIM
operating point; it supports neither cross-setting scaling, an amortized
training--inference break-even point, nor a quality--compute frontier.  The
small PI/Step-PI timing difference also cannot be interpreted as an inherent
acceleration or slowdown from the release schedule.

\FloatBarrier

\subsection{Descriptive Native-Route Quality Positioning}

The equal-dataset macro comparison in the main paper favors Step-PI over
LanPaint and PILOT on all five metrics.  Appendix~B specifies the method
attributes and frozen native routes.

\begin{table}[H]
\centering
\small
\begin{tabular}{llccccc}
\toprule
Dataset & Method & Masked L1 $\downarrow$ & Boundary L1 $\downarrow$ & Masked LPIPS $\downarrow$ & Boundary LPIPS $\downarrow$ & CLIP-Q $\uparrow$ \\
\midrule
AFHQ & Step-PI & 0.1694 & 0.0320 & 0.1379 & 0.0980 & 0.7109 \\
AFHQ & LanPaint & 0.1677 & 0.0439 & 0.1518 & 0.1279 & 0.6797 \\
AFHQ & PILOT & 0.1708 & 0.0456 & 0.1502 & 0.1224 & 0.6742 \\
AFHQ & SD-Inpaint & 0.1263 & 0.0272 & 0.1067 & 0.0656 & 0.7605 \\
AFHQ & BrushNet & 0.1396 & 0.0317 & 0.1209 & 0.0841 & 0.7542 \\
AFHQ & PixelHacker & 0.1211 & 0.0243 & 0.1201 & 0.0590 & 0.7333 \\
CelebA-HQ & Step-PI & 0.1278 & 0.0236 & 0.0690 & 0.0923 & 0.5719 \\
CelebA-HQ & LanPaint & 0.1342 & 0.0371 & 0.0771 & 0.1373 & 0.6052 \\
CelebA-HQ & PILOT & 0.1203 & 0.0371 & 0.0672 & 0.1158 & 0.5999 \\
CelebA-HQ & SD-Inpaint & 0.0893 & 0.0174 & 0.0461 & 0.0591 & 0.7136 \\
CelebA-HQ & BrushNet & 0.1052 & 0.0224 & 0.0552 & 0.0824 & 0.7008 \\
CelebA-HQ & PixelHacker & 0.0943 & 0.0156 & 0.0524 & 0.0698 & 0.6385 \\
Places2 & Step-PI & 0.1606 & 0.0317 & 0.1300 & 0.1184 & 0.4763 \\
Places2 & LanPaint & 0.1805 & 0.0502 & 0.1679 & 0.1740 & 0.3638 \\
Places2 & PILOT & 0.1879 & 0.0471 & 0.1539 & 0.1603 & 0.4310 \\
Places2 & SD-Inpaint & 0.1438 & 0.0253 & 0.1113 & 0.0951 & 0.4783 \\
Places2 & BrushNet & 0.1664 & 0.0335 & 0.1354 & 0.1366 & 0.4622 \\
Places2 & PixelHacker & 0.2218 & 0.0229 & 0.1256 & 0.0851 & 0.4919 \\
\bottomrule
\end{tabular}
\caption{Per-dataset external positioning for Step-PI and the five selected external methods, including independently audited LanPaint and PILOT overlays.}
\label{tab:external-absolute}
\end{table}

The first three rows within each dataset form the vanilla-SD1.5 training-free
lane; the remaining three rows are inpainting-trained references.  The advantage is
broad but not uniform across every dataset--metric cell:
Step-PI records better point estimates than LanPaint in 13 of 15 cells; the
exceptions are AFHQ Masked L1 and CelebA-HQ CLIP-Q.  It records better point
estimates than PILOT in 12 of 15 cells; the exceptions are CelebA-HQ Masked L1,
Masked LPIPS, and CLIP-Q.  Because methods retain different audited native
routes and compute budgets, these results provide descriptive system-level
positioning rather than an intervention-matched ranking.

\FloatBarrier

\subsection{Reproduction Scope and Numerical Execution}

\paragraph{Evaluated stack.}
The evaluated claim is configuration-locked cross-domain transfer within one
frozen vanilla-SD1.5 stack (original VAE, $512\!\times\!512$ inputs, and 50-step
DDIM).  The common controller configuration was developed only on small
CelebA-HQ pilots disjoint from Main35, then transferred without retuning to AFHQ
and Places2.  This evidence does not establish transfer across other backbones,
VAEs, samplers, resolutions, or denoising budgets.

\paragraph{Reproduction protocol.}
The reported experiments are reproducible as protocol-level reruns using the
fixed inputs, configurations, seeds, and analysis procedures described in
Appendix~B.  The persistent-state and joint-release five-seed studies are
separate paired inference-initialization studies on the fixed MS175 cohort;
neither is a repeated-seed expansion of Main35.  Static checks verify paired
image, mask, seed, latent, and noise identities, ground-truth isolation during
inference, exact known-region preservation, and resume logic.  Upon publication,
the implementation, experiment configurations, case lists, evaluation scripts,
and table/figure generation code will be released.

\paragraph{Numerical execution boundary.}
Independent CUDA executions can differ at the level of floating-point
arithmetic and are therefore not expected to produce byte-identical image
files.  Reproduction here means rerunning the specified protocol and recovering
the reported statistical conclusions, not requiring bitwise identity across
independent GPU processes.  Reported intervals characterize the case and/or
inference-initialization variation specified by each study; they do not estimate
low-level arithmetic variation across executions.  The one-field schedule
study is consequently interpreted at its intended scope---conditional
sensitivity on the evaluated cohort---rather than as a separate claim of
cross-execution numerical robustness.

\FloatBarrier

\end{document}